\documentclass[nohyperref]{article}

\usepackage{microtype}
\usepackage{graphicx}
\usepackage{subfigure}
\usepackage{booktabs} 

\usepackage{hyperref}

\usepackage{multirow}

\usepackage{float}

\usepackage[accepted]{icml2022}

\usepackage{amsmath}
\usepackage{amssymb}
\usepackage{mathtools}
\usepackage{amsthm}
\usepackage{adjustbox}
\usepackage{makecell}

\usepackage[capitalize,noabbrev]{cleveref}

\theoremstyle{plain}
\newtheorem{theorem}{Theorem}[section]
\newtheorem{proposition}{Proposition}[section]
\newtheorem{lemma}{Lemma}[section]
\newtheorem{corollary}{Corollary}[section]
\theoremstyle{definition}
\newtheorem{definition}{Definition}
\newtheorem{assumption}{Assumption}
\theoremstyle{remark}
\newtheorem{remark}{Remark}[section]

\newtheorem{example}{Example}[section]

\usepackage[textsize=tiny]{todonotes}

\usepackage{amsmath,amsfonts,bm}

\def\1{\bm{1}}

\def\ry{{\textnormal{y}}}

\def\rvx{{\mathbf{x}}}

\def\rvz{{\mathbf{z}}}

\def\rmA{{\mathbf{A}}}

\def\rmX{{\mathbf{X}}}

\def\rmZ{{\mathbf{Z}}}

\def\vg{{\bm{g}}}

\def\vu{{\bm{u}}}
\def\vv{{\bm{v}}}
\def\vw{{\bm{w}}}

\def\vy{{\bm{y}}}
\def\vz{{\bm{z}}}

\def\mI{{\bm{I}}}

\def\mW{{\bm{W}}}

\DeclareMathAlphabet{\mathsfit}{\encodingdefault}{\sfdefault}{m}{sl}
\SetMathAlphabet{\mathsfit}{bold}{\encodingdefault}{\sfdefault}{bx}{n}

\def\gA{{\mathcal{A}}}

\def\gC{{\mathcal{C}}}
\def\gD{{\mathcal{D}}}

\def\gF{{\mathcal{F}}}

\def\gH{{\mathcal{H}}}
\def\gI{{\mathcal{I}}}

\def\gL{{\mathcal{L}}}
\def\gM{{\mathcal{M}}}

\def\gO{{\mathcal{O}}}
\def\gP{{\mathcal{P}}}

\def\gR{{\mathcal{R}}}

\def\gT{{\mathcal{T}}}

\def\gW{{\mathcal{W}}}

\def\sP{{\mathbb{P}}}

\def\sR{{\mathbb{R}}}
\def\sS{{\mathbb{S}}}

\newcommand{\E}{\mathbb{E}}

\def\bxi{\boldsymbol{\xi}}

\def\bDelta{\boldsymbol{\Delta}}

\def\bG{\boldsymbol{G}}

\def\bSigma{\boldsymbol{\Sigma}}

\def\tvw{\widetilde{\vw}}
\def\hvw{\widehat{\vw}}
\def\bvw{\bar{\vw}}
\def\cvw{\check{\vw}}

\def\prox{\operatorname{Prox}}
\def\cprox{\operatorname{CProx}}

\newcommand{\infnorm}[1]{\|#1\|_{\infty}}

\newcommand{\twonorm}[1]{\|#1\|}
\newcommand{\LRtwonorm}[1]{\left\|#1\right\|}
\newcommand{\onenorm}[1]{\|#1\|_{1}}

\newcommand{\LRs}[1]{\left(#1\right)}
\newcommand{\LRm}[1]{\left[#1\right]}
\newcommand{\LRl}[1]{\left\{#1\right\}}

\newcommand{\inprod}[2]{\langle #1, #2\rangle}

\usepackage{pifont}
\icmltitlerunning{Fast Composite Optimization and Statistical Recovery in Federated Learning}

\begin{document}

\twocolumn[
\icmltitle{Fast Composite Optimization and Statistical Recovery in Federated Learning}



\icmlsetsymbol{equal}{*}

\begin{icmlauthorlist}
\icmlauthor{Yajie Bao}{sjtu}
\icmlauthor{Michael Crawshaw}{gmu}
\icmlauthor{Shan Luo}{sjtu}
\icmlauthor{Mingrui Liu}{gmu}
\end{icmlauthorlist}

\icmlaffiliation{sjtu}{School of Mathematical Sciences, Shanghai Jiao Tong University, Shanghai, China}
\icmlaffiliation{gmu}{Department of Computer Science, George Mason University, Fairfax, VA 22030, USA}

\icmlcorrespondingauthor{Mingrui Liu}{mingruil@gmu.edu}

\icmlkeywords{client Learning, ICML}

\vskip 0.3in
]



\printAffiliationsAndNotice{\textbf{This is a revised version to fix the imprecise statements about linear speedup from the ICML proceedings. We use another averaging scheme for the returned solutions in Theorem \ref{thm_sc} and \ref{thm_one_stage} to guarantee linear speedup when the number of iterations is large.}}  

\begin{abstract}
As a prevalent distributed learning paradigm, Federated Learning (FL) trains a global model on a massive amount of devices with infrequent communication. This paper investigates a class of composite optimization and statistical recovery problems in the FL setting, whose loss function consists of a data-dependent smooth loss and a non-smooth regularizer. Examples include sparse linear regression using Lasso, low-rank matrix recovery using nuclear norm regularization, etc. In the existing literature, federated composite optimization algorithms are designed only from an optimization perspective without any statistical guarantees. In addition, they do not consider commonly used (restricted) strong convexity in statistical recovery problems. We advance the frontiers of this problem from both optimization and statistical perspectives. From optimization upfront, we propose a new algorithm named \textit{Fast Federated Dual Averaging} for strongly convex and smooth loss and establish state-of-the-art iteration and communication complexity in the composite setting. In particular, we prove that it enjoys a fast rate, linear speedup, and reduced communication rounds. From statistical upfront, for restricted strongly convex and smooth loss, we design another algorithm, namely \textit{Multi-stage Federated Dual Averaging}, and prove a high probability complexity bound with linear speedup up to optimal statistical precision. Experiments in both synthetic and real data demonstrate that our methods perform better than other baselines. To the best of our knowledge, this is the first work providing fast optimization algorithms and statistical recovery guarantees for composite problems in FL.
\end{abstract}

\section{Introduction}
Federated Learning (FL) is a popular learning paradigm in distributed learning that enables a large number of clients to collaboratively learn a global model without sharing individual data \citep{mcmahan2017communication}. The most well-known algorithm in FL is called Federated Averaging (FedAvg). In each round, FedAvg samples a subset of devices and runs multiple steps of Stochastic Gradient Descent (SGD) on these devices in parallel, then the central server updates the global model by aggregation at the end of the communication round and broadcasts the updated model to clients. It has been verified that FedAvg achieves similar performance with fewer communication rounds compared with parallel SGD \citep{li2019convergence,stich2019local,woodworth2020minibatch}.

\begin{table*}[tb]
\centering
\caption{Comparison of related works under bounded heterogeneity (see Assumption \ref{assum_hetero}). (R)SC and GC refer to (restricted) strongly convex and general convex respectively. (R)SM refers to (restricted) smooth. N/A means not available. $K$: number of clients; $L$: smooth parameter; $\mu$: (restricted) strongly convex parameter; $\sigma^2$: variance of stochastic gradient; $\hvw$: global minimizer of \eqref{FCO}; $\hvw_{\text{Fast-FedDA}}$: returned solution from Algorithm \ref{alg:fast-FedDA} after running $T$ iterations on each client; $\hvw_{\text{MC-FedDA}}$: returned solution from Algorithm \ref{alg:multi-FedDA} after running $T$ iterations on each client; $\vw^{*}$: ground-truth solution \eqref{pop_form}; $\epsilon_{\text{stat}}$: optimal statistical precision in Proposition \ref{thm_stat}; $\tilde{\gO}$: hides logarithmic factors.}\label{table::comparison}
\resizebox{\textwidth}{!}{
\begin{tabular}{@{}llllllll@{}}
\toprule
Algorithm & Reference & \makecell[l]{Communication rounds\\
for linear speedup} & Problem & Conditions& \makecell[l]{Convergence rate for\\
$\phi(\hvw_{\text{Fast-FedDA}}) - \phi(\hvw)$} & \makecell[l]{Iteration complexity for\\ $\twonorm{\hvw_{\text{MC-FedDA}} - \vw^{*}}^2 \leq \epsilon_{\text{stat}}$} & Guarantee        \\ \midrule
FedAvg &\citep{woodworth2020minibatch} & $\gO(T^{1/2} K^{1/2})$  & unconstrained & SC, SM & $\gO(\sigma^2/(\mu K T))$ & N/A & Expectation      \\
 & & $\gO(T^{3/4} K^{3/4})$  & unconstrained & GC, SM & $\gO(\sigma/\sqrt{K T})$ & N/A                   & Expectation \\
SCAFFOLD & \citep{karimireddy2020scaffold}& $\tilde{\gO}(L/\mu)$ & unconstrained & SC, SM & $\gO(\sigma^2/(\mu K T))$ & N/A & Expectation \\
& & $\gO(T^{1/2}K^{1/2})$ & unconstrained & GC, SM & $\gO(\sigma/\sqrt{K T})$ & N/A & Expectation\\
FedDA &\citep{yuan2021federated} & $\gO(T^{3/4} K^{3/4})$  & composite & GC, SM & $\gO(\sigma/\sqrt{K T})$ & N/A     & Expectation      \\
Fast-FedDA & Theorem \ref{thm_sc} & $\gO(T^{1/2} K^{1/2})$ & composite & SC, SM   & $\gO(\sigma^2/(\mu K T))$& N/A & Expectation      \\
MC-FedDA & Theorem \ref{thm_multi_stage}  & $\gO(T^{1/2} K^{1/2})$ & composite & RSC, RSM  & $\tilde{\gO}(\sigma^2/(\mu K T))$  & $\tilde{\gO}(\sigma^2/(\mu K \epsilon_{\text{stat}}))$ & High probability \\ \bottomrule
\end{tabular}}
\end{table*}

Most of the research in FL mainly focuses on unconstrained smooth optimization problems without a regularizer and assumes each client has access to its local population distribution. However, people usually want the learned model to have some patterns, such as (group) sparsity and low rank. These desired patterns are usually achieved by solving composite optimization problems, e.g., LASSO \citep{tibshirani1996regression}, Graphical LASSO \citep{friedman2008sparse}, Elastic net \citep{zou2005regularization}, matrix completion \citep{candes2009exact}. So it is crucial to study how to solve these composite problems under the FL environment in the current big data era. This paper considers solving a composite optimization problem in the FL paradigm where only infrequent communication is allowed. In particular, we aim to solve
\begin{equation}\label{FCO}
    \min_{\vw \in \sR^p} \phi(\vw):= \sum_{k=1}^K\pi_k \gL_k(\vw) + h(\vw),
\end{equation}
where $\pi_k$ is the weight of the $k$-th client, $\sum_{i=1}^{K}\pi_k=1$, $\gL_k(\vw) = \E_{\xi \sim \gP_k}[f(\vw;\xi)]$ is the loss function evaluated at the $k$-th client, $\gP_k$ denotes the population distribution on the $k$-th client, and $h(\vw)$ is a non-smooth regularizer. The most related works which can solve this problem in FL setting are  \citet{yuan2021federated} and \citet{tran2021feddr}. \citet{yuan2021federated} proposed an algorithm called Federated Dual Averaging (FedDA). In one round, each client in FedDA performs dual averaging to update its primal and dual states for several steps; then, the server aggregates the dual states and updates the global primal state by a proximal step. However, they did not consider exploiting the strong convexity of the loss function and hence only ended up with a slow convergence rate (i.e., $\gO(1/\sqrt{T})$). \citet{tran2021feddr} considered non-convex loss and provided an algorithm that can converge to a point with small gradient mapping, but it does not have any global optimization guarantees. It remains unclear how to improve the convergence rate further when solving strongly convex composite problems in the FL setting. To answer this question, we propose a new algorithm, namely \textit{Fast Federated Dual Averaging} (\texttt{Fast-FedDA}), with provable fast rate, linear speedup and almost the same communication complexity achieved by FedAvg as in the unconstrained strongly convex case without regularizer \citep{woodworth2020minibatch, karimireddy2020scaffold}.

A fundamental assumption in FL literature is that it assumes that every client can have access to its local population distribution. However, it may not be the case in practice: each client usually only has access to its local empirical distribution \citep{negahban2012unified,agarwal2012fast,wainwright2019high}. This motivates us to consider a more challenging problem in FL: statistical recovery. It is devoted to recovering the ground-truth model parameter $\vw^{*}$ by only accessing the empirical distribution. It is much more difficult than the typical results in FL since we need to simultaneously deal with computational, statistical, and communication efficiency.
Statistical recovery is usually achieved through solving a composite optimization problem as
\begin{equation}\label{eq:FCO_recovery}
    \min_{\vw \in \sR^p} \phi(\vw):= \sum_{k=1}^K\pi_k \gL_k(\vw) + \lambda\gR(\vw),
\end{equation}
where $\gL_k(\vw) = \E_{\xi\sim \gD_k}[f(\vw;\xi)]$ is the empirical loss at $k$-th client \footnote{We use $\gL_k$ to denote the empirical loss in Section \ref{sec:recovery}, and denote the population loss in Section \ref{sec:optimization}.}, $\gD_k$ is the corresponding empirical distribution on the $k$-th client, $\lambda$ is regularization parameter, and $\gR(\cdot)$ is a non-smooth norm penalty. In addition, due to high dimensionality and small sample size in each client, the assumption of strong convexity might be demanding and unrealistic. Hence we further consider the broadly used restricted strong convexity (RSC) and restricted smooth (RSM) conditions in statistical recovery problems \citep{agarwal2012fast,wang2014optimal,loh2015regularized}.  

Distributed statistical recovery is an extensively studied topic in recent years \citep{lee2017communication,wang2017efficient,jordan2018communication,chen2020distributed}, but these works assume that all clients have the same data distribution, and they also assume that there is a closed-form solution for some non-trivial optimization subproblems. The unique features of FL are high heterogeneity in data among clients and local updates to solve these subproblems explicitly. Hence the algorithms and theoretical results of the literature mentioned above are not directly applicable in the FL regime. To address this issue, under RSC and RSM conditions, we first introduce an algorithm named \emph{Constrained Federated Dual Averaging} (\texttt{C-FedDA}) to solve a $\gR$-norm constrained subproblem. Then we introduce another algorithm called \emph{Multi-stage Constrained Federated Dual Averaging} (\texttt{MC-FedDA}), which calls \texttt{C-FedDA} in multiple stages with adaptively changing hyperparameter (e.g., shrinking radius of $\gR$-norm ball). In particular, after finishing one stage of \texttt{C-FedDA}, we use the output as a warm-start for the next stage.





A comparison of our results to related works is presented in Table \ref{table::comparison}. For more related work, please refer to Appendix \ref{appendix::relate_work}. We summarize our contributions in the following:
\begin{enumerate}
    \item For federated composite optimization problems with strongly convex and smooth loss, we propose the \texttt{Fast-FedDA} algorithm for solving~\eqref{FCO} by accessing data sampled from population distribution. Under the general bounded heterogeneity assumption, we show that \texttt{Fast-FedDA} enjoys linear speedup, and the communication complexity matches the lower bound of \texttt{FedAvg} for strongly convex problems without a regularizer \citep{karimireddy2020scaffold}.
    
    \item To obtain the statistical recovery results through solving \eqref{eq:FCO_recovery}, we propose an algorithm, namely \texttt{MC-FedDA} in the FL setting by accessing data sampled from the empirical distribution. Under RSC and RSM conditions, we prove that \texttt{MC-FedDA} enjoys optimal (high probability) convergence rate $\tilde{\gO}(\sigma^2\log(1/\delta)/(\mu T K))$ to attain statistical error bound. We find that the typical convergence rate in expectation in FL is insufficient for achieving statistical recovery guarantees, and this is the first high probability result for composite problems in FL.

    \item We conduct numerical experiments on linear regression with $\ell_1$ penalty, low-rank matrix estimation with nuclear norm penalty, and multiclass logistic regression with $\ell_1$ penalty. Both synthetic and real data show better performances of \texttt{Fast-FedDA} and \texttt{MC-FedDA} compared with other baselines.
\end{enumerate}


\paragraph{Notations.}
For a vector $\vw \in \sR^p$, we use $\twonorm{\vw}$ to denote the Euclidean norm. For a matrix $\mW \in \sR^{p_1\times p_2}$, we use $\|\mW\|_{\text{F}}$ to denote the Frobenius norm and use $\|\mW\|_{\text{nuc}}$ to denote the nuclear norm. 
For two real positive sequences $a_n$ and $b_n$, we write $a_n\lesssim b_n$ if there exists some positive constant $c$ such that $a_n\leq c b_n$. We use $a_n = \gO(b_n)$ to hide multiplicative absolute constant $c$ and also use $a_n = \tilde{\gO}(b_n)$ to hide logarithmic factors. 
In our paper, $\E_{\gP}$ means taking expectation with the randomness from true distribution $\gP$ and $\E_{\gD}$ means taking expectation with the randomness from empirical distribution $\gD$.

\section{Fast Federated Composite Optimization}\label{sec:optimization}
In this section, we focus on the composite optimization problem in FL environment for strongly convex and smooth loss. Given a user-specific loss function $f(\cdot): \sR^{p} \to \sR$, suppose there are $K$ clients, let $\gL_k(\vw) = \E_{\xi\sim\gP_k}[f(\vw;\xi)]$ be the local population loss and $\pi_k$ be the local weight for $k=1,...,K$. We consider the following composite problem
\begin{equation}\label{fed-co}
    \widehat{\vw} = \arg \min_{\vw \in \gW} \LRl{\sum_{k=1}^K \pi_k\gL_k(\vw) + h(\vw)},
\end{equation}
where $\gW \subseteq \sR^p$ is the domain and $h:\sR^p \to \sR$ is a non-smooth regularizer. From now on, we denote $\gL(\vw) = \sum_{k=1}^K \pi_k\gL_k(\vw)$ and write the global composite objective as $\phi(\vw) = \gL(\vw) + h(\vw)$.

\begin{assumption}\label{assum:local_obj}
The local loss functions $\gL_k$ for $k\in [K]$ are $L$-smooth and $\mu$-strongly convex, that is for any $\vw$, $\vw^{\prime} \in \gW$, there exist $0<\mu \leq L $ such that
\begin{equation*}
    \gL_k(\vw) - \gL_k(\vw^{\prime}) - \inprod{\nabla \gL_k(\vw^{\prime})}{\vw - \vw^{\prime}} \geq \frac{\mu}{2} \twonorm{\vw - \vw^{\prime}}^2
\end{equation*}
and
\begin{equation*}
    \gL_k(\vw)- \gL_k(\vw^{\prime}) - \inprod{\nabla \gL_k(\vw^{\prime})}{\vw - \vw^{\prime}} \leq \frac{L}{2} \twonorm{\vw - \vw^{\prime}}^2.
\end{equation*}
\end{assumption}

To solve the problem \eqref{fed-co} under strongly convex case, we propose a new algorithm named \emph{Fast Federated Dual Averaging} (Fast-FedDA) in Algorithm~\ref{alg:fast-FedDA}. The main difference between our algorithm and the FedDA algorithm in~\citet{yuan2021federated} is that we employed a different dual-averaging scheme in the local updates of Algorithm~\ref{alg:fast-FedDA} (line 9). In particular, we not only use information on history cumulative gradient $\vg_t^k$ as in~\citet{yuan2021federated} but also history model parameter $\tvw_t^k$ to leverage the strong convexity.

\begin{algorithm}[tb]
   \caption{Fast-FedDA($\vw_0$, $R$, $E$, $\mu$, $\gamma$, $a$)}
   \label{alg:fast-FedDA}
\begin{algorithmic}[1]
   \STATE {\bfseries Input:} Initial point $\vw_0$, constants $(\mu, \gamma, a)$ and synchronized set $\gI = \{t_r:0\leq r \leq R\}$ with $t_{r+1} = t_r + E$.
   \STATE \textbf{Initialize:} $\vw_0^k = \vw_0$ for $k\in [K]$, $\alpha_t = (t+a)^2$.
   \FOR{Round $r = 0$ {\bfseries to} $R$}
   \FOR{Client $k= 1$ {\bfseries to} $K$}
   \FOR{$t = t_r$ {\bfseries to} $t_{r+1}-1$}
   \STATE Query $\bG_t^k =\nabla f(\vw_t^k;\xi_t^k)$ for $\xi_t^k \sim \gP_k$.
   \STATE Compute $\vg_t^k = \vg_{t-1}^k+ \alpha_t\bG_t^k$.
   \IF{$t< t_{r+1}-1$}
    \STATE Update: $\vw_{t+1}^k = \prox_{t}(\vg_{t}^k - \mu\tvw_t^k/2)$ and $\tvw_{t+1}^k = \tvw_{t}^k + \alpha_{t+1}\vw_{t+1}^k$.
   \ELSE
   \STATE Send $\vg_{t_{r+1}-1}^k$ and $\tvw_{t_{r+1}-1}^k$ to the server.
   \ENDIF
   \ENDFOR
   \ENDFOR
    \STATE Server aggregates: $\vg_{t_{r+1}-1} = \sum_{k=1}^K\pi_k\vg_{t_{r+1}-1}^k$ and $\tvw_{t_{r+1}-1} = \sum_{k=1}^k\pi_k\tvw_{t_{r+1}-1}^k$.
    \STATE Server updates: $\vw_{t_{r+1}} = \prox_{t_{r+1}-1}(\vg_{t_{r+1}-1} - \mu\tvw_{t_{r+1}-1}/2)$ and $\tvw_{t_{r+1}} = \tvw_{t_{r+1}-1} + \alpha_{t_{r+1}}\vw_{t_{r+1}}$.
    \STATE Synchronization: $\vg_{t_{r+1}-1}^k \leftarrow \vg_{t_{r+1}-1}$ and $\tvw_{t_{r+1}}^k \leftarrow \tvw_{t_{r+1}}$.
   \ENDFOR
\end{algorithmic}
\end{algorithm}

\subsection{Fast Federated Dual Averaging}
 We begin with defining a \textit{proximal operator} $\prox_t(\vz)$ for $t\geq 0$ as the solution of the following problem:
\begin{equation*}
\min_{\vw \in \gW}\LRl{\inprod{\vw}{\vz - \gamma\vw_0} + \LRs{\frac{\mu A_t}{2}+\gamma}\frac{\twonorm{\vw}^2}{2} + A_t h(\vw)},
\end{equation*}
where $\vw_0$ is the initial point and $A_t = \sum_{i=0}^t \alpha_i$ is the summation of weights. For a loss function with \textit{strong convexity coefficient} $\mu > 0$, classical stochastic dual averaging \citep{tseng2008accelerated,nesterov2009primal,chen2012optimal} updates the model parameter by $\vw_{t+1} = \prox_{t}(\vg_t - \mu\tvw_t/2)$
where $\vg_t = \sum_{i=0}^t\alpha_i \nabla f(\vw_i;\bxi_i)$ is the weighted summation of past stochastic gradients and $\tvw_t = \sum_{i=0}^t \alpha_i\vw_i$ is the weighted summation of past solutions.
Denote the synchronized step set by $\gI = \{t_r ~|~ t_r=rE \text{ for } 0\leq r \leq R\}$, 
where $t_R = (R+1)E = T$. Similar to FedAvg \citep{mcmahan2017communication} and FedDA \citep{yuan2021federated}, a natural idea to develop a federated dual averaging algorithm for strongly convex loss includes following local updates and server aggregation and update:
\begin{itemize}
    \item The $k$-th client updates its local solution by
    \begin{align*}
        \vw_{t+1}^k = \prox_{t}\LRs{\vg_t^k - \mu\tvw_t^k/2},
    \end{align*}
    for $t_r\leq t \leq t_{r+1}-1$.
    
    \item The server updates the global solution by
    \begin{equation*}
        \vw_{t_{r+1}} = \prox_{t_{r+1}-1}\LRs{\sum_{k=1}^K\pi_k (\vg_t^k - \mu \tvw_t^k/2)}.
    \end{equation*}
\end{itemize}
For ease of reference, the detailed procedure of Fast-FedDA is summarized in Algorithm \ref{alg:fast-FedDA}. The definitions of $\vg_t^k$ and $\tvw_t^k$ are provided in line 7 and 9 respectively.

\subsection{Main Results for Fast-FedDA}
To establish the convergence results for Fast-FedDA, we impose the following assumptions on the regularizer $h$, loss function and stochastic gradients.

\begin{assumption}\label{assum_regu_sc}
The regularizer $h: \gW \to \sR$ is a closed convex function.
\end{assumption}

\begin{assumption}\label{assum_hetero}
The global loss function is $\Lambda$-smooth, which means $\twonorm{\nabla\gL(\vw) - \nabla\gL(\vw^{\prime})} \leq \Lambda \twonorm{\vw - \vw^{\prime}}$. In addition, there exists some positive constant $H$ such that $\twonorm{\nabla \gL(\vw) - \nabla \gL_k(\vw)}\leq H$ for any $\vw \in \gW$ and $k = 1,...,K$.
\end{assumption}

\begin{assumption}\label{assum_grad_noise}
The stochastic gradient sampled from the local population distribution $\gP_k$ satisfies that: for any $\vw \in \gW$, it holds that $\E_{\xi\sim \gP_k}[\nabla f(\vw;\xi)] = \nabla \gL_k(\vw)$ and $\E_{\xi\sim\gP_k}[\twonorm{\nabla f(\vw;\xi) -  \nabla \gL_k(\vw)}^2] \leq \sigma^2$.
\end{assumption}

Assumption \ref{assum_regu_sc} is very common in composite optimization literature \citep{tseng2008accelerated,nesterov2009primal,xiao10a}. We use Assumption \ref{assum_hetero} to bound the heterogeneity between clients, which also appears in \citet{woodworth2020minibatch,yuan2020federated}.
The next theorem provides the convergence rate of Fast-FedDA in expectation, and the proof is deferred to Appendix \ref{proof:thm_sc}.

\begin{theorem}\label{thm_sc}
Under Assumptions \ref{assum:local_obj}-\ref{assum_grad_noise}, we assume the domain is bounded by $\rho>0$, that is $\gW = \{\vw \in \sR^p:\twonorm{\vw}\leq \rho\}$. We choose $\alpha_t = (t+a)^2$ and $\gamma = 2\mu a^3$ for $a \geq 4L/\mu$ in Algorithm \ref{alg:fast-FedDA} and let $A_T = \sum_{t=0}^T\alpha_t$. Considering $\hvw_{\text{Fast-FedDA}} = \sum_{t=0}^T\alpha_t\prox_{t+1}(\sum_{k=1}^K \pi_k(\vg_t^k + \mu\tvw_{t}^k/2))/A_T$, for $T \geq a$, it satisfies that
\begin{equation}\label{rate_sc}
    \begin{aligned}
    \E_{\gP}[\phi(\hvw_{\text{Fast-FedDA}}) - \phi(\hvw)] \lesssim \frac{\mu a^3 B}{T^3} + \frac{\bar{\sigma}^2 }{\mu T} + \frac{L E \sigma^2 }{\mu^2 T^2}&\\
    + \frac{L E^2(H + \Lambda\rho)^2}{\mu^2 T^2}&,
\end{aligned}
\end{equation}
where $\bar{\sigma}^2 = \sum_{k=1}^K\pi_k^2\sigma^2$ and $B = \twonorm{\vw_0-\hvw}^2$.
\end{theorem}
\begin{remark}
Here we apply the same averaging scheme with weight $\alpha_t = (t+a)^2$ in \citet{stich2019local}. The first two terms match the results of Theorem 2.2 in \citet{stich2019local} for non-composite problems. And the last two terms are incurred from infrequent communication. Now considering the equal-weighted case, that is $\pi_1 = \cdots = \pi_K = 1/K$, the weighted variance is given by $\bar{\sigma}^2 = \sigma^2/K$. In \eqref{rate_sc}, we may choose $E^2 \lesssim \sigma^2 \mu T/ (L K(H + \Lambda\rho)^2)$, and then other terms in \eqref{rate_sc} will be dominated by $\sigma^2/(\mu K T)$ when $T \gtrsim \mu L^{1/2} a^{3/2} K^{1/2}/\sigma$. Hence the convergence rate attains \emph{linear speedup} with respective to $K$. Meanwhile, the communication complexity of Fast-FedDA is $\gO(T^{1/2}K^{1/2})$, which matches the lower bound for FedAvg (see Theorem II in \citet{karimireddy2020scaffold}). Under the same bounded heterogeneity assumption, \citet{yuan2021federated} only considered the quadratic loss. In this vein, we investigate a more general loss function in Theorem \ref{thm_sc}.
\end{remark}

\section{Fast Federated Statistical Recovery}\label{sec:recovery}
In this section, we consider the statistical recovery via composite optimization in FL framework. Let $\gP_k$ for $k=1,2,...,K$ be the unknown local population distributions, then the ``true parameter'' is defined as
\begin{equation}\label{pop_form}
    \vw^{*} = \arg \min_{\vw \in \gW} \sum_{k=1}^K \pi_k \E_{\xi \sim \gP_k} [f(\vw;\xi)].
\end{equation}
Denote the i.i.d. dataset sampled from $\gP_k$ by $\{\xi_i: i \in \gH_k \text{ and }|\gH_k| = n_k\}$. We may obtain the sparse/low-rank estimator of $\vw^{*}$ through solving the following composite problem
\begin{equation}\label{central_problem}
    \widehat{\vw} = \arg \min_{\vw \in \gW} \LRl{\sum_{k=1}^K \pi_k\gL_k(\vw) + \lambda\gR(\vw)},
\end{equation}
where $\gL_k(\vw) = \sum_{i\in \gH_k} f(\vw;\xi_i)/n_k$ is the local empirical loss function and $\gR(\cdot)$ is a non-smooth norm regularizer. Here we use $\gD_k$ to denote the \emph{empirical distribution} on the $k$-th client, which means $\gL_k(\vw) = \E_{\xi\sim \gD_k} [f(\vw;\xi)]$.

\subsection{Illustrative Examples}\label{sec:examples}
In this subsection, we take two well known examples to illustrate the statistical recovery problems in FL.
\begin{example}[Sparse Linear Regression]\label{example_lasso}
The linear model in each client is given by
\begin{equation*}
    \ry_{i}^{k} = (\rvx_i^k)^{\top}\vw^{*} + \varepsilon_i^k \quad \text{for } i\in [n_k]\text{ and }k \in [K],
\end{equation*}
where the covariate $\rvx_i^k$ follows some unknown distribution $\gP_k$ and the noise $\varepsilon_i^k \sim N(0, 1)$ is independent of $\rvx_i^k$. We assume the true regression coefficient $\vw^{*}$ is $s$-sparse, that is $\|\vw^{*}\|_0 = s$, and $s\ll p$. Let $\pi_k = \frac{n_k}{N}$, then our goal is to solve the following federated Lasso problem
\begin{equation*}
    \hvw = \arg\min_{\vw \in \gW}\frac{1}{2N}\sum_{k=1}^K\sum_{i=1}^{n_k} (\ry_i^k - (\rvx_i^k)^{\top}\vw)^2 + \lambda \onenorm{\vw},
\end{equation*}
where $\lambda$ is the regularization parameter and $\gW = \{\vw: \twonorm{\vw}\leq \rho\}$.
\end{example}

\begin{example}[Low-Rank Matrix Estimation]\label{example_low_rank}
Let $\mW^{*} \in \sR^{p_1\times p_2}$ be an unknown matrix with low rank $r^{*} \ll \min\{p_1,p_2\}$. For each client, the response variable $\vy_i^{k}$ and covariate matrix $\rmX_i^k$ are linked to the unknown matrix via
\begin{equation*}
    \ry_i^k = \inprod{\rmX_i^k}{\mW^{*}} + \varepsilon_i^k \quad \text{for } i\in [n_k]\text{ and }k \in [K],
\end{equation*}
where $\rmX_i^k$ is sampled from some unknown distribution $\gP_k$ and the noise $\varepsilon_i^k \sim N(0, 1)$ is independent of $\rvx_i^k$. Let $\pi_k = \frac{n_k}{N}$, then our goal is to solve the following federated trace regression problem
\begin{equation*}
    \widehat{\mW} = \arg\min_{\mW \in \gW} \frac{1}{2N}\sum_{k=1}^K\sum_{i=1}^{n_k}(\vy_i^k - \inprod{\rmX_i^k}{\mW})^2 + \lambda\|\mW\|_{\text{nuc}},
\end{equation*}
where $\lambda$ is the regularization parameter and $\gW = \{\mW: \|\mW\|_{\text{F}}\leq \rho\}$.
\end{example}

\subsection{Restricted Strong Convexity and Smoothness}
To develop the techniques for the statistical properties of regularization in FL, we introduce the definition of decomposable regularizer \citep{negahban2012unified}.
 \begin{definition}\label{def_decompose}
 Given a pair of subspaces in $\sR^{p}$ such that $\gM \subseteq \bar{\gM}$, a norm regularizer $\gR$ is decomposable with respect to $(\gM,\bar{\gM}^{\perp})$ if
 \begin{equation*}
     \gR(\vw+\vv) = \gR(\vw) + \gR(\vv)\quad \text{for all }\vw \in \gM \text{ and }\vv \in \bar{\gM}^{\perp}.
 \end{equation*}
 The subspace Lipschitz constant with respect to the subspace $\bar{\gM}$ is defined by
 \begin{equation*}
     \Psi(\bar{\gM}) := \sup_{\vu \in \bar{\gM}\setminus \{\boldsymbol{0}\}}\frac{\gR(\vu)}{\twonorm{\vu}}.
 \end{equation*}
 \end{definition}

\begin{assumption}\label{assum_regu}
The regularizer $\gR(\cdot)$ is a norm with dual $\gR^{*}(\cdot)$, which satisfies $\gR^{*}(\cdot)\leq \twonorm{\cdot}\leq \gR(\cdot)$. There is a pair of subspace $\gM \subseteq \bar{\gM}$ such that the regularizer decomposes over $(\gM, \bar{\gM}^{\perp})$. Moreover, we assume $\vw^{*} \in \gM$.
\end{assumption}
\begin{remark}
$\bar{\gM}$ usually encodes structural information of the regularizer. For example, for sparse linear model in Example \ref{example_lasso}, the subspace is defined by $\bar{\gM} \equiv \gM := \LRl{\vw \in \sR^p| w_j = 0 \text{ for }j\in \sS}$ for some subset $\sS\subseteq [p]$. Correspondingly, the subspace Lipschitz constant is given by $\Psi(\bar{\gM}) = \sqrt{s}$, where $s$ is the cardinality of the support set $\sS$. And the regularizer is $\onenorm{\cdot}$, whose dual norm is $\infnorm{\cdot}$. Clearly, Assumption \ref{assum_regu} is satisfied for sparse linear model since $\infnorm{\cdot} \leq \twonorm{\cdot} \leq \onenorm{\cdot}$. In Example \ref{example_low_rank}, Assumption \ref{assum_regu} is also satisfied. Due to space limit, we refer to \citet{negahban2012unified} for more details about $\gM$ and $\bar{\gM}$ in low-rank matrix estimation.
\end{remark}

In the high-dimensional setting ($p>n_k$), it is usually hard to guarantee the strong convexity for the local empirical loss $\gL_k$. Therefore, we consider the restricted strong convexity in Assumption \ref{assum_local_rsc}, which is widely used in statistical recovery literature \citep{agarwal2012fast,wang2014optimal,loh2015regularized,cai2020cost}. For $k=1,..,K$, denote the first-order Taylor series expansion of $\gL_k(\vw)$ around $\gL_k(\vw^{\prime})$ by
\begin{align*}
    \gT_k(\vw,\vw^{\prime}) = \gL_k(\vw) - \gL_k(\vw^{\prime}) - \inprod{\nabla \gL_k(\vw^{\prime})}{\vw - \vw^{\prime}}.
\end{align*}
\begin{assumption}\label{assum_local_rsc}
The local loss functions $\gL_k$ for $k=1,...,K$ are convex and satisfy the restricted strongly convex (RSC) condition, that is for $\vw$, $\vw^{\prime} \in \gW$, there exist $\mu > 0$ and $\tau_k \geq 0$ such that
\begin{equation*}
    \gT_k(\vw,\vw^{\prime}) \geq \frac{\mu}{2} \twonorm{\vw - \vw^{\prime}}^2 - \tau_k \gR^2(\vw - \vw^{\prime}).
\end{equation*}
\end{assumption}
From Assumption \ref{assum_local_rsc}, the global loss function $\gL$ also satisfies the RSC condition: for any $\vw$, $\vw^{\prime} \in \gW$
\begin{equation}\label{global_RSC}
    \begin{aligned}
    \gT(\vw,\vw^{\prime})&=\gL(\vw) - \gL(\vw^{\prime}) - \inprod{\nabla \gL(\vw^{\prime})}{\vw - \vw^{\prime}}\\
    &\geq \frac{\mu}{2} \twonorm{\vw - \vw^{\prime}}^2 - \tau \gR^2(\vw - \vw^{\prime}),
    \end{aligned}
\end{equation}
where $\tau = \sum_{k=1}^K\pi_k\tau_k$. We also introduce an analogous notion of restricted smoothness.
\begin{assumption}\label{assum_local_RSM}
The local loss functions $\gL_k$ for $k=1,...,K$ satisfy the restricted smooth (RSM) condition, that is for any $\vw$, $\vw^{\prime} \in \gW$ there exist $L >0$ and $\nu_k \geq 0$ such that
\begin{equation*}
    \gT_k(\vw,\vw^{\prime}) \leq \frac{L}{2} \twonorm{\vw - \vw^{\prime}}^2 + \nu_k\gR^2(\vw - \vw^{\prime}).
\end{equation*}
\end{assumption}
Similarly, under Assumption \ref{assum_local_RSM}, the global loss $\gL$ satisfies the RSM condition with coefficient $L$ and $\nu = \sum_{k=1}^K\pi_k \nu_k$.

With the decomposable regularizer $\gR$ and the RSC condition \eqref{global_RSC} for $\gL$, the statistical recovery results via solving the composite problem \eqref{central_problem} has been extensively investigated in the past decade (see \citet{negahban2012unified,wainwright2019high} and references therein). We present the optimal statistical error of global estimator $\widehat{\vw}$ in the following proposition, which is a direct result of Corollary 1 in \citet{negahban2012unified} or Theorem 9.19 in \citet{wainwright2019high}. The error bound in Proposition \ref{thm_stat} is also the target precision to achieve optimal statistical recovery.
\begin{proposition}\label{thm_stat}
Under Assumptions \ref{assum_regu} and \ref{assum_local_rsc}. If $\tau\Psi^2(\bar{\gM}) \leq \frac{\mu}{64}$ holds, with choice $\lambda \geq 2 \gR^{*}(\nabla \gL(\vw^{*}))$ in \eqref{central_problem}, the statistical error of $\hvw$ can be bounded by
\begin{equation*}
    \twonorm{\widehat{\vw} - \vw^{*}} \leq \frac{3\Psi(\bar{\gM})\lambda_{\text{opt}}}{\mu} \text{ and }\gR(\widehat{\vw} - \vw^{*}) \leq \frac{12\Psi^2(\bar{\gM})\lambda_{\text{opt}}}{\mu}.
\end{equation*}
\end{proposition}
Therefore, the \emph{optimal statistical precision} $\epsilon_{\text{stat}} = \Psi(\bar{\gM})\lambda_{\text{opt}}/\mu$ can be achieved by choosing optimal regularization parameter $\lambda_{\text{opt}} \geq 2 \gR^{*}(\nabla \gL(\vw^{*}))$.

\subsection{Constrained Federated Dual Averaging}\label{subsec:cfda}

In light of Proposition \ref{thm_stat}, we aim to estimate the ground-truth $\vw^{*}$ defined in \eqref{pop_form} by solving the following composite problem:
\begin{equation}\label{eq:opt_problem}
    \hvw_{\text{opt}}= \arg\min_{\vw \in \gW}\LRl{\gL(\vw) + \lambda_{\text{opt}} \gR(\vw)}
\end{equation}
where $\gL(\vw) = \sum_{k=1}^K\pi_k \gL_k(\vw)$ and $\lambda_{\text{opt}} \geq 2 \gR^{*}(\nabla \gL(\vw^{*}))$. Let $\hvw_{\text{Fed}}$ be the output of a federated algorithm, and we hope $\twonorm{\hvw_{\text{Fed}} - \vw^{*}}^2$ can attain the optimal statistical precision $\epsilon_{\text{stat}}$ with iteration complexity $\gO(\bar{\sigma}^2/(\mu \epsilon_{\text{stat}}))$. Similar to \citet{woodworth2020local,yuan2021federated}, Fast-FedDA also has a drawback. To guarantee the fast convergence rate, the final estimator of Algorithm \ref{alg:fast-FedDA} takes the weighted average of all iterations. However, we cannot obtain this estimator in the FL setting, since the server only has access to the solution $\vw_{t}$ for $t\in \gI$. To address this issue, we first propose a new algorithm named \textit{Constrained Federated Dual Averaging} (C-FedDA) in Algorithm \ref{alg:FedDA} in subsection~\ref{subsec:cfda}. In addition, to achieve optimal statistical recovery guarantees, we propose another algorithm named \emph{Multi-stage Constrained Federated Dual Averaging} (MC-FedDA) in Algorithm~\ref{alg:multi-FedDA} in subsection~\ref{subsec:mcfda}, which calls Algorithm~\ref{alg:FedDA} as a subroutine. We provide convergence rate for Algorithm~\ref{alg:FedDA} and statistical recovery results for Algorithm~\ref{alg:multi-FedDA}, both in high probability.

\begin{algorithm}[tb]
   \caption{C-FedDA($\vw_0$, $R$, $E$, $\epsilon_0$, $\mu$, $\gamma$, $a$, $\lambda$)}
   \label{alg:FedDA}
\begin{algorithmic}[1]
   \STATE {\bfseries Input:} Initial point $\vw_0$, constants $(\epsilon_0,\mu,\gamma, a)$ and synchronized set $\gI = \{t_r:1\leq r \leq R\}$.
   \STATE \textbf{Initialize:} $\tvw_0 = \bvw_0 = \vw_0$, $\alpha_r = (r+a)^2$.
   \FOR{Round $r = 0$ \textbf{to} $R$}
   \FOR{Client $k= 1$ {\bfseries to} $K$}
   \FOR{$t = t_r$ {\bfseries to} $t_{r+1}-1$}
   \STATE Query $\bG_t^k =\nabla f(\vw_t^k;\xi_t^k)$ for for $\xi_t^k \sim \gD_k$.
    \STATE Update $\vg_t^k = \vg_{t-1}^k+\alpha_r\bG_t^k$.
   \IF{$t < t_{r+1}-1$}
   \STATE $\vw_{t}^k = \cprox_{r}(\vg_{t}^k - \mu E \tvw_r/2;\vw_0,\epsilon_0,\lambda)$.
   \ELSE
   \STATE Send $\vg_{t_{r+1}-1}^k$ to the server.
   \ENDIF
   \ENDFOR
   \ENDFOR
   \STATE Server aggregates: $\vg_{t_{r+1}-1} = \sum_{k=1}^K\pi_k\vg_{t_{r+1}-1}^k$.
    \STATE Server updates:
    \begin{align*}
        \bvw_{r+1} = \cprox_{r}(\vg_{t_{r+1}-1} - \mu E\tvw_r/2;\vw_0,\epsilon_0,\lambda)
    \end{align*}
    and $\tvw_{r+1} = \tvw_r + \alpha_{r+1}\bvw_{r+1}$.
    \STATE Synchronization: $\vg_{t_{r+1}-1}^k \leftarrow \vg_{t_{r+1}-1}$.
   \ENDFOR
\end{algorithmic}
\end{algorithm}

For the ease of representation, we define a \textit{constrained proximal operator} $\cprox_r(\vz; \vw_0, \epsilon_0, \lambda)$ for $r\geq 0$ as the solution of the following constrained problem:
\begin{align*}
\min_{\vw \in \gW(\epsilon_0;\vw_0)}\Big\{\inprod{\vw}{\vz - \gamma E\vw_0} + \LRs{\frac{\mu A_r}{2}+\gamma}\frac{E\twonorm{\vw}^2}{2}&\\
+ A_r E\lambda \gR(\vw)\Big\}&,
\end{align*}
where $\gW(\epsilon_0;\vw_0):= \{\vw \in \gW ~|~ \gR(\vw - \vw_0)\leq \epsilon_0\}$. Let $\tvw_r = \sum_{j=0}^{r}\alpha_j \bvw_j$ be the sum of past solutions obtained on the \emph{server}. In the $r$-th round, each client updates the weighted cumulative gradient as $\vg_t^k = \vg_{t-1}^k + \alpha_r \bG_t^k$ and updates the local solution by
\begin{equation*}
    \vw_{t+1}^k = \cprox_{r}\LRs{\vg_t^k - \frac{\mu E}{2}\tvw_r;\vw_0,\epsilon_0,\lambda},
\end{equation*}
for $t_r\leq t \leq t_{r+1}-1$. At the end of the $r$-th round, the server updates the global solution by
\begin{equation*}
    \bvw_{r+1} = \cprox_{r}\LRs{\sum_{k=1}^K \pi_k \vg_{t_{r+1}-1}^k - \frac{\mu E}{2}\tvw_r;\vw_0,\epsilon_0,\lambda},
\end{equation*}
and the weighted cumulative variable $\tvw_{r+1} = \tvw_{r}+\alpha_{r+1}\bvw_{r+1}$. The details of C-FedDA is stated in Algorithm \ref{alg:FedDA}, which can output a weighted estimator $\hvw_{\text{Fed}} = \sum_{r=0}^R\alpha_r\bvw_{r+1}/A_R$ with provable convergence rate. To cope with the RSC and RSM conditions, we introduce the following light-tailed condition to perform high-probability analysis \citep{duchi2012randomized,chen2012optimal,lan2012optimal}.
\begin{assumption}\label{assum_grad_noise_tail}
The stochastic gradient sampled from the local empirical distribution $\gD_k$ satisfies: for any $\vw \in \gW$, it holds that $\E_{\xi\sim \gD_k}[\nabla f(\vw;\xi)] = \nabla \gL_k(\vw)$ and
\begin{equation*}
    \E_{\xi \sim \gD_k}\LRm{\exp\LRs{\twonorm{\nabla f(\vw;\xi) - \nabla \gL_k(\vw)}^2/\sigma^2}} \leq \exp(1).
\end{equation*}
\end{assumption}

\begin{theorem}\label{thm_one_stage}
Under Assumptions \ref{assum_hetero} and \ref{assum_regu}- \ref{assum_grad_noise_tail}, we assume the initial point satisfies $\gR(\vw_0 - \hvw)\leq \epsilon_0$ and $\gW = \{\vw \in \sR^p:\twonorm{\vw}\leq \rho\}$ for $\rho > 0$. We choose $\alpha_r = (r+a)^2$ and $\gamma = 2\mu a^3$ with for some $a \geq 4 L/\mu$ in Algorithm \ref{alg:FedDA} and let $A_R = \sum_{r=0}^R \alpha_r$. With probability at least $1- \delta$, for $R \geq a$, the output $\hvw_{\text{C-FedDA}} = \sum_{r=0}^R\alpha_r\bvw_{r+1}/A_R$ satisfies that
\begin{equation}\label{rsc_fast_rate}
    \begin{aligned}
   &\phi\LRs{\hvw_{\text{C-FedDA}}} - \phi(\hvw_{\text{opt}}) \lesssim \frac{\mu a^3 E^3 \epsilon_0^2}{T^3} + \frac{\bar{\sigma}^2\log(1/\delta)}{\mu T}\\
   &+ \frac{L E\sigma^2\log^2(1/\delta)}{\mu^2 T^2} + \frac{LE^2(\Lambda\rho + H)^2}{\mu^2 T^2}\\
    &+ \frac{\epsilon_0 \bar{\sigma} \sqrt{\log(1/\delta)}}{\sqrt{T}} + (\tau + \nu)\epsilon_0^2
\end{aligned}
\end{equation}
where $\bar{\sigma}^2 = \sum_{k=1}^K\pi_k^2\sigma^2$.
\end{theorem}
This theorem provides a \emph{high probability} convergence result for C-FedDA in Algorithm \ref{alg:FedDA}. The proof of Theorem \ref{thm_one_stage} can be found in Appendix \ref{proof:thm_one_stage}.
\begin{remark}
For the R.H.S. of~\eqref{rsc_fast_rate}, there are 6 terms. The 1st and 2nd terms come from the parallel dual averaging, the 3rd and 4th terms are due to skipped communication, the 5th term comes from concentration inequality, and the 6-th term $(\tau + \nu)\epsilon_0^2$ in \eqref{rsc_fast_rate} incurs an additional error regarding the tolerances in RSC and RSM conditions. The 5th term $\bar{\sigma}\sqrt{\epsilon_0\log(1/\delta)/T}$ is the best-known high probability rate of centralized dual averaging \citep{xiao10a,lan2012optimal,chen2012optimal}. 
By choosing $E^2 \lesssim \bar{\sigma}^2 \mu T/ (L(H + \Lambda\rho)^2)$ for Algorithm \ref{alg:FedDA}, the discrepancy (the 3rd and 4th term) from local updates will be dominated by the concentration bound (the 5th term) when $T \gtrsim \max\{\mu^{5/2}a^3 \bar{\sigma}^{2}, \bar{\sigma}^2/(\mu^2\epsilon_0^2)\}$.
\end{remark}

\begin{algorithm}[t]
   \caption{Multi-stage C-FedDA}
   \label{alg:multi-FedDA}
\begin{algorithmic}
   \STATE {\bfseries Input:} Initial point $\hvw_0$, number of stages $M$, $\{R_m,E_m\}_{m=0}^{M-1}$ and initial regularization parameter $\lambda_0$.
   \FOR{Stage $m = 0$ {\bfseries to} $M-1$}
   \STATE Update: $\lambda_m = 2^{-m}\lambda_0$ and $\epsilon_{m} = 108\Psi^2(\bar{\gM})\lambda_{m}/\mu$.
   \STATE Update estimator by calling C-FedDA
    \begin{equation*}
       \widehat{\vw}_{m+1} =\text{C-FedDA}(\hvw_{m}, R_m, E_{m}, \epsilon_{m}, \mu, \gamma, a, \lambda_m).
    \end{equation*}
   \ENDFOR
\end{algorithmic}
\end{algorithm}

\subsection{Multi-stage Constrained Federated Dual Averaging}\label{subsec:mcfda}
 To reduce the error brought from the RSC and RSM conditions, the first attempt is solving~\eqref{eq:opt_problem} by directly using shrinking domain technique \citep{iouditski2014primal,hazan2011beyond,lan2012optimal,liu2018fast} according to $\gR(\cdot)$-norm. 
 In each stage, we use the output of the previous stage as the initial point and shrink the radius of the $\gR(\cdot)$-norm ball in C-FedDA. In particular, we need to guarantee that $\gR^2(\hvw_{m} - \hvw_{\text{opt}})$ is also reduced with high probability through controlling $(\phi(\hvw_m) - \phi(\hvw_{\text{opt}}))/\lambda_{\text{opt}}$ at the $m$-th stage (see Lemma \ref{lemma_iter_cone}), since we need to make sure that $\hvw_{\text{opt}}$ always lies into the ball with high probability. However, it can be only decreased up to $(\tau+\nu)\gR^2(\hvw_{m-1}-\hvw_{\text{opt}})/\lambda_{\text{opt}}$
 according to the last term in \eqref{rsc_fast_rate}, which could be very large since $\lambda_{\text{opt}}$ is usually very small. This indicates that we cannot directly employ shrinking domain technique for solving~\eqref{eq:opt_problem}.
 
 To address this issue, our solution is motivated by the homotopy continuation strategy \citep{xiao2013proximal,wang2014optimal}: we select a decreasing sequence of the regularization parameter \footnote{Here we set $1/2$ as the contraction rate for technique convenience. In practice, we may choose more flexible non-increasing sequence $\lambda_m$.} $\lambda_m = \lambda_0 \cdot 2^{-m}$, where $ \lambda_{\text{opt}} < \lambda_0$ and $\lambda_M = \lambda_{\text{opt}}$. At the $m$-th stage, we call Algorithm \ref{alg:FedDA} to solve the following subproblem
 \begin{equation*}
     \min_{\vw \in \gW(\hvw_{m-1},\epsilon_m)} \LRl{\gL(\vw) + \lambda_m\gR(\vw)},
 \end{equation*}
 where $\hvw_{m-1}$ is the output of previous stage and $\epsilon_m$ is the current radius. 
By shrinking both the radius and regularization parameter in each stage,  a final estimator with optimal statistical precision can be obtained. We present the detailed procedure of \textit{Multi-stage Constrained Federated Dual Averaging} (MC-FedDA) in Algorithm \ref{alg:multi-FedDA}.

\subsection{Main Results for MC-FedDA}
In this subsection, we present the statistical recovery results of the algorithm MC-FedDA. The proof of Theorem \ref{thm_multi_stage} is deferred to Appendix \ref{appendix::multi_stage}.

\begin{assumption}\label{assum_coef}
There exists some constant $C>0$, such that the averaged RSC and RSM coefficients satisfy $C(\tau + \nu)\Psi^2(\bar{\gM}) \leq \mu$.
\end{assumption}


\begin{theorem}\label{thm_multi_stage}
Under the same conditions in Theorem \ref{thm_one_stage} and Assumption \ref{assum_coef}. We assume the initial point satisfies $\gR(\hvw_0 - \hvw_{\text{opt}}) \leq 84 \Psi^2(\bar{\gM}) \lambda_0/\mu$ and choose $E_m$ and $R_m$ such that $E_m^2 \lesssim \bar{\sigma}^2 \mu T_m/ (L(H + \Lambda\rho)^2)$ and
\begin{align*}
    T_m = \gO\LRs{\frac{\Psi^4(\bar{\gM})\bar{\sigma}^2\log(2M/\delta)}{\mu^2 \epsilon_m^2}}
\end{align*}
for $T_m = E_m R_m$. When Algorithm \ref{alg:multi-FedDA} terminates ($M = \log_2(\lambda_0/\lambda_{\text{opt}}) + 1$), with probability\footnote{The randomness is from the empirical distribution $\gD = \{\gD_k:k=1,...,K\}$.} at least $1 - \delta$, the total number of iterations $T = \sum_{m=0}^M T_m$ is no more than (up to a constant factor)
\begin{equation}\label{complexity}
    \begin{aligned}
      \frac{\bar{\sigma}^2(\log_2(\lambda_0/\lambda_{\text{opt}})+1)}{\lambda_{\text{opt}}^2}\log\LRs{\frac{\log_2(\lambda_0/\lambda_{\text{opt}}) + 1}{\delta}}.
    \end{aligned}
\end{equation}
Let $\hvw_{\text{MC-FedDA}} = \hvw_M$ from Algorithm \ref{alg:multi-FedDA}, we can guarantee that
\begin{equation*}\label{excess_bound}
    \phi(\hvw_{\text{MC-FedDA}}) - \phi(\hvw_{\text{opt}})\leq \frac{\Psi^2(\bar{\gM})\lambda_{\text{opt}}^2}{\mu}.
\end{equation*}
In addition, the estimation error can be bounded by
\begin{equation*}
    \twonorm{\hvw_{\text{MC-FedDA}} - \vw^{*}} \leq \frac{4\Psi(\bar{\gM})\lambda_{\text{opt}}}{\mu}.
\end{equation*}
\end{theorem}

\begin{remark}
Notice that $\epsilon_{\text{stat}} = \Psi(\bar{\gM})\lambda_{\text{opt}}/\mu$ converges to 0 as the total sample size $N$ tends to infinity. Let $\epsilon = \mu \epsilon_{\text{stat}}^2$. According to \eqref{complexity}, if the total number of iterations satisfies $T = \sum_{m=0}^{M-1} T_m = \widetilde{\gO}(\Psi^2(\bar{\gM})\bar{\sigma}^2/(\mu\epsilon))$, then we are guaranteed that $\phi(\hvw_{\text{MC-FedDA}}) - \phi(\hvw_{\text{opt}})\leq \epsilon$. Up to some logarithmic factors and the subspace Lipschitz constant $\Psi^2(\bar{\gM})$, this is equivalent to the linear speedup convergence rate under the equal-weighted case ($\bar{\sigma}^2 = \sigma^2/K$). Moreover, with the choice $E_m = \tilde{\gO}(T_m^{1/2}/K^{1/2})$, the total communication complexity is bounded by $\sum_{m=0}^{M-1} T_m^{1/2}K^{1/2} = \tilde{\gO}(T^{1/2} K^{1/2})$. In fact, the total complexity is mainly due to the complexity of the last stage.
\end{remark}

Next, we illustrate the implications of Theorem \ref{thm_multi_stage} through Example \ref{example_lasso} and \ref{example_low_rank} in subsection \ref{sec:examples}.

\paragraph{Sparse Linear Regression.} Under some regular conditions, the RSC and RSM coefficients in each client are given by $\tau_k = c\log p/n_k$ and $\nu_k = c\log p/n_k$ for some absolute constant $c$ (see \citet{agarwal2012fast,loh2015regularized}). With the weight choice $\pi_k = \frac{n_k}{N}$, we have $\tau = \nu = c K\log p/N$. If the total sample size $N$ and the number of clients $K$ satisfies $s K \lesssim \mu N$, then Assumption \ref{assum_coef} will be satisfied. According to Proposition \ref{thm_stat}, we need to choose the regularization parameter such that $\lambda_{\text{opt}} \geq c \sqrt{\log p/ N}$ to guarantee the optimal statistical convergence rate $\twonorm{\hvw - \vw^{*}} = \gO(\sqrt{s\log p/N})$ with high probability \citep{raskutti2009minimax,ye2010rate}. Therefore, to attain the optimal statistical convergence rate, the iteration complexity in Theorem \ref{thm_multi_stage} is given by $\tilde{\gO}(N/K)$.

\paragraph{Low-Rank Matrix Estimation.} In this case, the subspace Lipschitz constant is $\Psi(\bar{\gM}) = \sqrt{r^{*}}$. Under some regular conditions, the averaged RSC and RSM coefficients are both $\tau = \nu = c (p_1 \vee p_2) K/ N$ \citep{agarwal2012fast,wainwright2019high}. If the total sample size $N$ and the number of clients $K$ satisfies $r^{*} K (p_1 \vee p_2) \lesssim \mu N$, then Assumption \ref{assum_coef} will be satisfied. To achieve optimal statistical convergence rate $\|\widehat{\mW} - \mW^{*}\|_{\text{F}} = \gO(\sqrt{r^{*}(p_1 \vee p_2)\log (p_1 \vee p_2)/N})$ with high probability \citep{koltchinskii2011nuclear}, we choose the regularization parameter as $\lambda \geq c \sqrt{(p_1 \vee p_2)\log (p_1 \vee p_2)/N}$. Thus the iteration complexity in Theorem \ref{thm_multi_stage} will be $\tilde{\gO}(N/(K(p_1 \vee p_2)))$.

\section{Numerical Experiments}\label{sec:experiment}
In this section, we investigate the empirical performance of our proposed method with four experiments: two with synthetic data and two with real world data. For Example \ref{example_lasso} and \ref{example_low_rank}, we generate heterogeneous synthetic data for 64 clients, and each client containing 128 independent samples. For federated sparse logistic regression, we use the Federated EMNIST \citep{caldas2019leaf} dataset of handwritten letters and digits. We compare our proposed three algorithms \texttt{Fast-FedDA}, \texttt{C-FedDA} and \texttt{MC-FedDA} with Federated Mirror Descent (\texttt{FedMiD}) and Federated Dual Averaging (\texttt{FedDA}) algorithms introduced in \citet{yuan2021federated}. The detailed parameter tuning of the experiments in this section is provided in Appendix \ref{sec:realexp}.

\begin{figure*}[tb]
    \centering
    \includegraphics[width=0.9\linewidth]{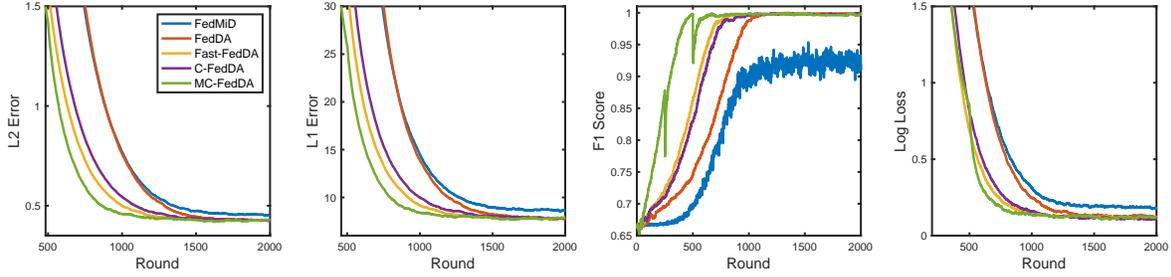}
    \caption{Recovery results for federated sparse linear regression problem with $p = 1024$ and $s = 512$. Except \texttt{FedMiD}, other methods nearly achieve perfect support recovery. Our proposed three algorithms show faster numerical convergence in four metrics.}
    \label{fig:lasso}
\end{figure*}

\begin{figure*}[tb]
    \centering
    \includegraphics[width=0.9\linewidth]{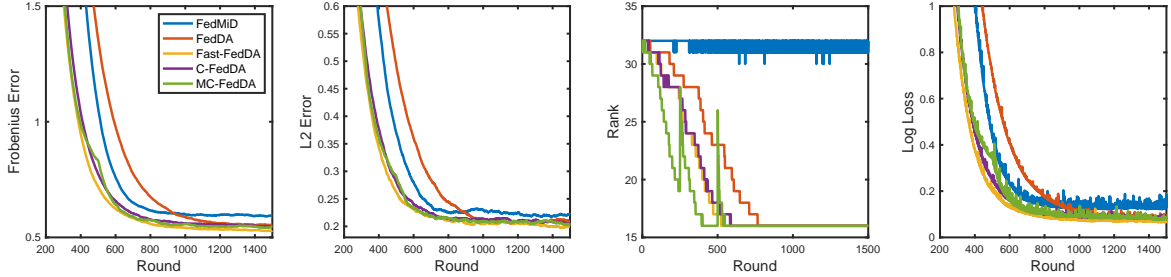}
    \caption{Recovery results for federated low-rank matrix estimation problem with $p_1= p_2 = 32$ and $r^{*} = 16$. Except \texttt{FedMiD}, other methods all recover the true rank of $\mW^{*}$. Our proposed three methods also show faster convergence than other two baselines.}
    \label{fig:nuclear}
\end{figure*}

\paragraph{Federated Sparse Linear Regression.}
In this experiment, we conduct experiments to Example \ref{example_lasso} on synthetic data. The true sparse regression coefficient is $\vw^{*} = (\mathbf{1}_{s}^{\top} \mathbf{0}_{p-s}^{\top})^{\top}$. In the $k$-th client, we first generate a heterogeneity vector $\boldsymbol{\delta}_k$ from $N(\mathbf{0}, \mI_{p\times p})$. The covariate is generated according to $\rvx_i^k = \boldsymbol{\delta}_k + \rvz_i^k$ for $i=1,2,...,n_k$, where $\rvz_i^k$ is independently sampled from $N(\mathbf{0}, \bSigma)$. The $(i,j)$-th element of the covariance matrix $\bSigma$ is given by $\sigma_{i,j} = 0.5^{|i-j|}$ for $1\leq i,j \leq p$. Then the response variable $\ry_i^k$ is generated accordingly. At each round, we sample 10 clients to conduct local updates and the number of local updates is $K=10$. In this experiment, the batch size is 10 and the regularization parameter is $\lambda = 0.5^5$.
To evaluate the performance of different methods, we record $\ell_2$ error, $\ell_1$ error, $F_1$ score of support recovery and training loss after each round and results are reported in Figure \ref{fig:lasso}.

\paragraph{Federated Low-Rank Matrix Estimation.}
In this subsection, we conduct experiments to Example \ref{example_low_rank} on synthetic data. The $p$ by $p$ true low-rank matrix is given by $\mW^{*} = \text{diag}(\mathbf{1}_{r^{*}}, \mathbf{0}_{p-r^{*}})$. At the $k$-th client, we first construct a heterogeneity matrix $\rmZ_k \in \sR^{p\times p}$ with each entry independently sampled from $N(0,1)$. Then we generate the covariate matrix by $\rmX_i^k = \rmZ_k + \rmA_i^k$, where each entry of $\rmA_i^k$ is also independently sampled from $N(0,1)$. As with the previous experiment, 10 clients are sampled each round to conduct local updates and $K=10$. In this experiment, the batch size is also 10 and the regularization parameter is $\lambda = 0.1$. We choose estimation error in Frobenius norm, $\ell_2$ norm (operator norm), recovery rank, and training loss to evaluate performances of different algorithms. The results are plotted in Figure \ref{fig:nuclear}.

From the results in Figure \ref{fig:lasso} and \ref{fig:nuclear}, we can see that our proposed three algorithms show faster convergence than \texttt{FedDA} and \texttt{FedMiD}, which is consisting with our linear speedup results. Except \texttt{FedMiD}, other methods nearly achieve perfect support recovery. As we expected, the evaluation metrics of \texttt{MC-FedDA} converge to the same values with other algorithms. It is worthwhile noting that $F_1$ score of \texttt{MC-FedDA} already converges to 1 after the first stage. The reason is that the regularization parameter is larger, which tends to output more sparse solution.

\paragraph{Federated Sparse Logistic Regression.}
We also provide experimental results on real world data, namely the Federated EMNIST dataset \citep{caldas2019leaf}. This dataset is a modification of the EMNIST dataset \citep{cohen2017emnist} for the federated setting, in which each client's dataset consists of all characters written by a single author. In this way, the data distribution differs across clients. The complete dataset contains 800K examples across 3500 clients. We train a multi-class logisitc regression model on two versions of this dataset: EMNIST-10 (digits only, 10 classes), and EMNIST-62 (all alphanumeric characters, 62 classes). Following \citep{yuan2021federated}, we use only $10\%$ of the samples, which is sufficient to train a logistic regression model. Our subsampled EMNIST-10 dataset consists of 367 clients with an average of 99 examples each, while EMNIST-62 consists of 379 clients with an average of 194 examples each.
For both experiments, we use a batch size of $25$, a regularization parameter $\lambda = 10^{-4}$, and we sample 36 clients to perform local updates at each communication round. For EMNIST-10, each sampled client performs $K = 40$ updates per communication round for $R = 15000$ rounds. For EMNIST-62, $K = 10$ and $R = 75000$. Comparisons of algorithms for EMNIST-62 and EMNIST-10 are shown in Figure~\ref{fig:emnist_10} and Figure \ref{fig:emnist_62} (Appendix \ref{sec:realexp}). We can see that our algorithms (\texttt{Fast-FedDA}, \texttt{C-FedDA}) outperforms baselines (\texttt{FedDA} and \texttt{FedMid}) in terms of convergence speed on both training and test performance.

\vspace*{-0.1in}
\section{Conclusion}
This paper investigates the composite optimization and statistical recovery problem in FL. For the composite optimization problem, we proposed a fast dual averaging algorithm (\texttt{Fast-FedDA}), in which we prove linear speedup for strongly convex loss. For statistical recovery, we proposed a multi-stage constrained dual averaging algorithm (\texttt{MC-FedDA}). Under restricted strongly convex and smooth assumption, we provided a high probability iteration complexity to attain optimal statistical precision, equivalent to the linear speedup result for strongly convex case. Several experiments on synthetic and real data are conducted to verify the superior performance of our proposed algorithms over other baselines.

\vspace*{-0.1in}
\section*{Acknowledgements}
We thank the anonymous reviewers for their valuable comments. 
Michael Crawshaw and Mingrui Liu are supported in part by a grant from George Mason University. Shan Luo's research is supported by National Program on Key Basic Research Project, NSFC Grant No. 12031005 (PI: Qi-Man Shao, Southern University of Science and Technology). Computations were run on ARGO, a research computing cluster provided by the Office of Research Computing at George Mason University (URL: https://orc.gmu.edu). The work of Yajie Bao was done when he was virtually visiting Mingrui Liu's research group in the Department of Computer Science at George Mason University.




\bibliography{ref_paper}
\bibliographystyle{icml2022}

\newpage
\appendix
\onecolumn

\section{Related Work}\label{appendix::relate_work}
\paragraph{Federated Learning.} As an active research area, a tremendous amount of research has been devoted to investigating the theory and application of FL. The most popular algorithm in FL is the so-called Federated Averaging (\texttt{FedAvg}) proposed by \citet{mcmahan2017communication}. For strongly convex problems, \citet{stich2019local} provided the first convergence analysis of \texttt{FedAvg} in a homogeneous environment and showed that the communication rounds can be reduced up to a factor of $\gO(\sqrt{T K})$ without affecting linear speedup. Then \citet{li2019convergence,li2020federated} investigated the convergence rate in a heterogeneous environment. ~\citet{pmlr-v119-karimireddy20a} introduced a stochastic controlled averaging for FL to learn from hetergeneous data. 
\citet{stich2019error,khaled2020tighter} improved the analysis and showed $\gO(K \text{poly}\log(T))$ rounds is sufficient to achieve linear speedup. Recently, \citet{yuan2020federated} proposed an accelerated FedAvg algorithm, which requires $\gO(K^{1/3}\text{poly}\log(T))$ to attain linear speedup. Recently, \citet{li2021statistical} investigated the statistical estimation and inference problem for local SGD in FL. However, \citet{li2021statistical} focused on the unconstrained smooth statistical optimization, but we considered a different problem with non-smooth regularizer aiming to recover the sparse/low-rank structure of ground-truth model. For the strongly convex finite-sum problem, \citet{mitra2021linear} proposed an algorithm named \texttt{FedLin} based on the variance reduction technique and obtained the linear convergence rate. However, their analysis and algorithm are not applicable in the composite setting, and they do not consider statistical recovery at all. A more recent work \citet{spiridonoff2021communication} showed that the number of rounds can be independent of $T$ under homogeneous setting. Recently, there is a line of work focusing on analyzing nonconvex problems in FL~\cite{yu2019parallel,yu2019linear,basu2019qsparse,haddadpour2019local}. This list is by no means complete due to the vast amount of literature in FL. For a more comprehensive survey, please refer to~\cite{kairouz2019advances} and reference therein.

\paragraph{Distributed Statistical Recovery.} With the increasing data size, statistical recovery in the distributed environment is a hot topic in recent years. These works focus on the homogeneous setting. \citet{lee2017communication} proposed an one-shot debiasing method and required each client solve a composite problem using its own data. Other one-shot methods for different tasks can be found in \citet{battey2018distributed,bao2021one,zhu2021least}. Motivated by the approximated Newton's method \citep{shamir2014communication}, \citet{wang2017efficient} proposed a multi-round algorithm, where each client only needs to compute gradients and the server solves a shifted $\ell_1$ penalized problem. Meanwhile, \citet{jordan2018communication} developed Communication-efficient Surrogate Loss (CSL) framework for more general $\ell_1$-penalized problems. A series of statistical recovery problems based on CSL scheme has also been studied \citep{liu2019distributed,chen2020distributed,tu2021byzantine}.

\section{Proof of Main Results}
\subsection{Concentration Inequalities for Martingale Differences}
Let $\LRl{\bxi_i \in \sR^p}_{i=1}^{\infty}$ be a sequence of martingale differences with respective to the filtration $\LRl{\gF_i}_{i=1}^{\infty}$. It satisfies that $\E[\bxi_i|\gF_{i}] = \boldsymbol{0}$ and the light-tail condition $\E_{\gD}[\exp(\twonorm{\bxi_i}^2/\sigma^2)|\gF_{i}] \leq \exp(1)$ for some $\sigma > 0$. Under this condition, it follows from Jensen's inequality that
\begin{equation*}
    \exp(\E_{\gD}[\twonorm{\bxi_i}^2|\gF_{i}]/\sigma^2) \leq \E_{\gD}[\exp(\twonorm{\bxi_i}^2/\sigma^2)|\gF_{i}] \leq \exp(1).
\end{equation*}
Hence we have $\E_{\gD}[\twonorm{\bxi_i}^2|\gF_{i}] \leq \sigma^2$. 

The following three lemmas are used throughout in our proof, we defer the proof of Lemma \ref{lemma_concen_norm} to Section \ref{appendix::shao}. 
\begin{lemma}[Lemma 5 in \cite{duchi2012randomized}]\label{lemma_Delta_concentration}
Under the assumption of Theorem \ref{thm_one_stage}, for any positive and non-decreasing sequence $\{a_t\}_{t=0}^{\infty}$, we have
\begin{equation*}
    \sum_{t=0}^{T}\frac{\twonorm{\bxi_t}^2}{a_t}\geq \sum_{t=0}^{T}\frac{\E_{\gD}[\twonorm{\bxi_t}^2]}{a_t} + \max \left\{\frac{8 \sigma^{2} \log (1 / \delta)}{a_{0}}, 16 \sigma^{2} \sqrt{\sum_{t=0}^{T} \frac{\log (1 / \delta)}{a_{i}^{2}}}\right\}
\end{equation*}
holds with probability at most $\delta \in (0,1)$.
\end{lemma}
\begin{lemma}[Lemma 6 in \cite{lan2012optimal}]\label{lemma_concen_prod}
Under the assumption of Theorem \ref{thm_one_stage}, for any sequence $\{\vw_t\}_{t=0}^{\infty}$ such that $\vz_t$ is $\gF_{t-1}$-measurable, we have
\begin{equation*}
    \sum_{t=0}^T \inprod{\vz_t}{\bxi_t} \geq \sqrt{3\log(1/\delta)}\LRs{\sum_{t=0}^T \twonorm{\vz_t}^2}^{1/2}
\end{equation*}
holds with probability at most $\delta \in (0,1)$.
\end{lemma}
The following lemma is a martingale's version of Lemma 3.1 in \citet{he2000parameters}, and the proof is deferred to Section \ref{appendix::shao}.

\begin{lemma}\label{lemma_concen_norm}
If $\E[\twonorm{\bxi_i}^2|\gF_{i}] < \infty$, then for any $x>0$ it holds that,
\begin{equation}\label{concen_norm}
    \sP\LRl{\LRtwonorm{\sum_{i=1}^{t} \bxi_i}\geq x\LRs{B_t + \LRs{\sum_{i=1}^t \twonorm{\bxi_i}^2}^{1/2}}} \leq 8\exp(-x^2/8),
\end{equation}
where $B_t = (\sum_{i=1}^t \E(\twonorm{\bxi_i}^2))^{1/2}$.
\end{lemma}
From now on, we use $\gF_{t}$ to denote the $\sigma$-algebra generated by prior sequence $\{\vw_i^k:0\leq i\leq t,\ 1\leq k\leq K\}$.

\subsection{Proof of Theorem \ref{thm_sc}}\label{proof:thm_sc}
Let $\vg_t = \sum_{k=1}^K\pi_k \vg_t^k$, $\cvw_t =  \sum_{k=1}^K \pi_k \vw_t^k$ and $\tvw_t = \sum_{k=1}^K \pi_k \tvw_t^k = \sum_{i=0}^t \alpha_i \cvw_i$, we define a virtual sequence:
\begin{equation}
    \vw_{t+1} = \arg\min_{\vw\in \gW}\LRl{\inprod{\vw}{\vg_{t} - \frac{\mu}{2}\tvw_t - \gamma\vw_0} + \LRs{\frac{A_t\mu}{2} + \gamma}\frac{\twonorm{\vw}^2}{2} + A_t h(\vw)}.
\end{equation}
which can be also equivalently written as
\begin{equation}\label{shadow_seq}
    \vw_{t+1} =  \arg\min_{\vw\in \gW}\LRl{\inprod{\vw}{\vg_{t}} + \frac{\mu}{4}\sum_{i=0}^t\alpha_i\twonorm{\vw-\cvw_i}^2 + \frac{\gamma}{2}\twonorm{\vw-\vw_0}^2 + A_t h(\vw)}.
\end{equation}
According to Algorithm \ref{alg:fast-FedDA}, $\vw_{t+1}$ is exactly the solution updated by the server for $t+1\in \gI$. Next we define a pseudo distance between $\vw$ and $\vw^{\prime}$ at the $t$-th step as
\begin{equation}\label{pseudo_dist}
    \begin{aligned}
    D_{t}(\vw;\vw^{\prime}) &= \inprod{\vw-\vw^{\prime}}{\vg_{t-1}} + \frac{\mu}{4}\sum_{i=0}^{t-1}\alpha_i \LRs{\twonorm{\vw-\cvw_i}^2 - \twonorm{\vw^{\prime}-\cvw_i}^2}\\
    &+ \frac{\gamma}{2} \LRs{\twonorm{\vw-\vw_0}^2 - \twonorm{\vw^{\prime}-\vw_0}^2} + A_{t-1} (h(\vw)-h(\vw^{\prime})).
    \end{aligned}
\end{equation}
Let $\vg_{-1} = \boldsymbol{0}$, $A_{-1} = 0$ and $\sum_{i=0}^{-1} = 0$, we have $D_{0}(\vw;\vw_0) = \frac{\gamma_0}{2} \twonorm{\vw-\vw_0}^2$ for any $\vw \in \gW$. In addition, \eqref{shadow_seq} also implies that $D_t(\vw;\vw_t) \geq 0$ for any $\vw\in \gW$.
The next lemma provide the one-step induction relation of Algorithm \ref{alg:FedDA}, which is crucial to the proof of convergence rate. The proof of Lemma \ref{lemma_one_step} is deferred to Section \ref{proof:lemma_lemma_one_step}.
\begin{lemma}[One-Step Induction Relation]\label{lemma_one_step}
Under the conditions of Theorem \ref{thm_sc}, it holds that
\begin{equation}\label{one_step_bound}
    \begin{aligned}
    \alpha_t[\phi(\vw_{t+1})-\phi(\widehat{\vw})] \leq & D_{t}(\widehat{\vw};\vw_{t})- D_{t+1}(\widehat{\vw};\vw_{t+1})+ \alpha_t\inprod{\bDelta_t}{\hvw - \vw_{t}} + \frac{\alpha_t^2\twonorm{\bDelta_t}^2}{2(A_t\mu  + 2\gamma - 2L\alpha_t)}\\
    &\qquad + \alpha_t\LRs{\frac{\mu}{2}\sum_{k=1}^K\pi_k\twonorm{\vw_t^k - \cvw_t}^2 + \frac{3L}{2}\sum_{k=1}^K\pi_k\twonorm{\vw_t^k - \vw_{t}}^2},
\end{aligned}
\end{equation}
where $\bDelta_t = \sum_{k=1}^K\pi_k (\nabla f(\vw_i^k; \xi_i^k) - \nabla \gL_k(\vw_i^k))$.
\end{lemma}
We impose the following lemma to bound the discrepancy caused by skipped communication, and the proof is deferred to Section \ref{appen:discre_1}.
\begin{lemma}\label{lemma_hetero_E}
Under the conditions in Theorem \ref{thm_sc}, we have
\begin{equation*}
    \E_{\gP}[\twonorm{\vw_t^k - \vw_t}^2|\gF_{t}],\ \E_{\gP}[\twonorm{\vw_t^k - \cvw_t}^2|\gF_{t}] \leq \frac{4 E\sigma^2 \alpha_t^2}{(\mu A_t/2 + \gamma)^2} + \frac{4E^2(H^2 + \Lambda^2 \rho^2)\alpha_t^2}{(\mu A_t/2 + \gamma)^2}.
\end{equation*}
\end{lemma}
\begin{proof}[Proof of Theorem \ref{thm_sc}]
We first note that $\E[\bDelta_t|\gF_t] = \boldsymbol{0}$ and
\begin{align*}
    \E_{\gP}[\twonorm{\bDelta_t}^2|\gF_{t}] = \sum_{k=1}^K \pi_k^2 \E_{\gP}[\twonorm{\nabla f(\vw_i^k; \xi_i^k) - \nabla \gL_k(\vw_i^k)}^2|\gF_{t}]\leq \sum_{k=1}^K \pi_k^2 \sigma^2 = \bar{\sigma}^2,
\end{align*}
where the second equality follows from the independence between different clients. Taking conditional expectation on the both sides of \eqref{one_step_bound} results in
\begin{equation}\label{mid_E_indu}
    \begin{aligned}
    &\alpha_t\E_{\gP}[\phi(\vw_{t+1})-\phi(\widehat{\vw})|\gF_t]\\
    \leq & \E_{\gP}[D_{t}(\widehat{\vw};\vw_{t})- D_{t+1}(\widehat{\vw};\vw_{t+1})|\gF_t] +\alpha_t\inprod{\E_{\gP}[\bDelta_t|\gF_t]}{\hvw - \vw_{t}} + \frac{\alpha_t^2\E_{\gP}[\twonorm{\bDelta_t}^2|\gF_t]}{2(A_t\mu  + 2\gamma - 2L\alpha_t)}\\
     & + \alpha_t\LRl{\frac{L+\mu}{2}\sum_{k=1}^K\pi_k\E_{\gP}[\twonorm{\vw_t^k - \cvw_t}^2|\gF_t] + L\sum_{k=1}^K\pi_k\E_{\gP}[\twonorm{\vw_t^k - \vw_{t}}^2|\gF_t]}\\
     \leq & \E_{\gP}[D_{t}(\widehat{\vw};\vw_{t})- D_{t+1}(\widehat{\vw};\vw_{t+1})|\gF_t]+ \frac{\alpha_t^2\bar{\sigma}^2}{2(A_t\mu  + 2\gamma - 2L\alpha_t)}\\
     &+ \frac{3L+\mu}{2}\alpha_t^3\LRs{\frac{4 E\sigma^2}{(\mu A_t/2 + \gamma)^2} + \frac{4E^2(H^2 + \Lambda^2\rho^2)}{(\mu A_t/2 + \gamma)^2}}.
    \end{aligned}
\end{equation}
In the second inequality of \eqref{mid_E_indu}, we used Lemma \ref{lemma_hetero_E}.
By substituting $\gamma = 2\mu a^3$ and $\alpha_t = (t+a)^2$ with $a \geq 2L/\mu$, we have
\begin{align*}
    \mu A_t + 2\gamma - 2L\alpha_t &= \mu \sum_{s = 0}^{t}(s+a)^2 +2\gamma - 2L\alpha_t\\
    &= \frac{\mu t(t+1)(2t+1)}{6} + \mu a t(t+1) + \mu a^2 t + 2\mu a^3 - 2Lt^2 - 4 a Lt - 2La^2\\
    &\geq \frac{\mu t^3}{3} + \mu a^3.
\end{align*}
It implies that for $T \geq a$
\begin{align*}
    \sum_{t=0}^{T} \frac{\alpha_t^2\bar{\sigma}^2}{2(A_t\mu + 2\gamma - 2L\alpha_t)} &= \sum_{t=0}^{T} \frac{3(t+a)^4\bar{\sigma}^2}{2\mu t^3+ 2\mu a^3}=\sum_{t=0}^{T} \frac{3(t^4+ 4at^3 + 6a^2t^2 + 4a^3t + a^4)\bar{\sigma}^2}{2\mu t^3+ 2\mu a^3}\\
    &\leq \LRs{\frac{T(T+1)}{\mu} + \frac{6 a T}{\mu} + \frac{9 a^2\log(T+1)}{\mu} + \frac{6 a^3}{\mu T} + \frac{2a^4}{\mu T^2} + \frac{2 a}{\mu}}\bar{\sigma}^2\\
    &\leq \frac{5 T(T+1)\bar{\sigma}^2}{\mu},
\end{align*}
and
\begin{align*}
    \sum_{t=0}^{T}\frac{\alpha_t^3}{(\mu A_t/2 + \gamma)^2} &\leq \sum_{t=0}^{T}\frac{4(t+a)^6}{(\mu(t^3/3 + at^2 + a^2 t + a^3))^2}\\
    &\leq \sum_{t=0}^{T}\frac{36(t+a)^6}{(\mu(t + a)^3)^2} \leq \frac{36 T}{\mu^2}.
\end{align*}
In addition, it follows from $D_{T+1}(\widehat{\vw};\vw_{T+1}) \geq 0$ and $D_{0}(\widehat{\vw};\vw_{0}) = \gamma_0\twonorm{\hvw - \vw_0}/2$ that
\begin{equation*}
    \begin{aligned}
    \sum_{t=0}^T \LRl{D_{t}(\widehat{\vw};\vw_{t})- D_{t+1}(\widehat{\vw};\vw_{t+1})} = &D_{0}(\widehat{\vw};\vw_{0}) - D_{T+1}(\widehat{\vw};\vw_{T+1})\\
    \leq& \frac{\gamma_0}{2}\twonorm{\hvw-\vw_0}^2.
    \end{aligned}
\end{equation*}
Let $B = \twonorm{\vw_0 - \hvw}^2$, telescoping \eqref{mid_E_indu} from time $t=0$ to $t = T$ gives rise to
\begin{align*}
    \frac{1}{A_T}\sum_{t=0}^{T}\alpha_t \E_{\gP}[\phi(\vw_{t+1})-\phi(\widehat{\vw})]\leq& \frac{\gamma\twonorm{\hvw-\vw_0}^2}{A_T} + \frac{1}{A_T}\sum_{t=0}^{T} \frac{\alpha_t^2\bar{\sigma}^2}{2(A_t\mu + 2\gamma - 2L\alpha_t)}\\
    + & \frac{3L+\mu}{2A_T}\sum_{t=0}^T \alpha_t^3\LRs{\frac{4 E\sigma^2}{(\mu A_t/2 + \gamma)^2} + \frac{4E^2(H^2 + \Lambda^2 + \mu^2 \rho^2)}{(\mu A_t/2 + \gamma)^2}}\\
    \lesssim &\frac{\mu a^3 B}{T^3} + \frac{\bar{\sigma}^2 }{\mu T} + \frac{L E \sigma^2 }{\mu^2 T^2} + \frac{L E^2(H^2 + \Lambda^2\rho^2)}{\mu^2 T^2},
\end{align*}
Thus the result follows from Jensen's inequality and the convexity of $\phi(\cdot)$.
\end{proof}

\subsection{Proof of Theorem \ref{thm_one_stage}}\label{proof:thm_one_stage}
Similar to the proof of Theorem \ref{thm_one_stage}, we define the following pseudo distance at the $r$-th communication step
\begin{equation}
    \begin{aligned}
    D_{r}(\vw;\vw^{\prime}) &= \inprod{\vg_{t_r - 1}}{\vw - \vw^{\prime}} + \frac{\mu E}{4}\sum_{j=0}^{r-1}\alpha_j (\twonorm{\vw - \bvw_j}^2 - \twonorm{\vw^{\prime} - \bvw_j}^2)\\
    &+ \frac{\gamma E}{2}(\twonorm{\vw - \bvw_0}^2 - \twonorm{\vw^{\prime} - \bvw_0}^2) + E A_{r-1} h(\vw),
    \end{aligned}
\end{equation}
where $\vg_{t_r-1} = \sum_{k=1}^K \pi_k \vg_{t_r-1}^k$.
Let $\vg_{-1} = \boldsymbol{0}$ and $\sum_{j = 0}^{-1} = 0$, then we have $D_{0}(\hvw;\bvw_0) = \frac{\gamma E}{2}\twonorm{\hvw - \bvw_0}^2 \leq \frac{\gamma}{2}\epsilon_0$. The following lemma characterizes the one round progress of Algorithm \ref{alg:FedDA}, and the proof is deferred to Section \ref{proof:lemma_one_step_R}.
\begin{lemma}[One-Step Induction Relation]\label{lemma_one_step_R}
Under the conditions of Theorem \ref{thm_one_stage}, we have
\begin{equation}\label{indu_R}
    \begin{aligned}
    E\alpha_r [\phi(\bvw_{r+1})-\phi(\widehat{\vw})] &\leq  D_{r}(\widehat{\vw};\bvw_{r})- D_{r+1}(\widehat{\vw};\bvw_{r+1})+ \alpha_r\sum_{i=t_r}^{t_{r+1}-1}\inprod{\bDelta_i}{\widehat{\vw} - \bvw_{r}} + 20E\alpha_r(\tau + \nu)\epsilon_0^2\\
    &+\frac{\alpha_r^2\twonorm{\sum_{i=t_r}^{t_{r+1}-1}\bDelta_i}^2}{2(A_r E \mu + 2\gamma E - 2(L+\mu)\alpha_r E)} + \frac{3L+2\mu}{2}\alpha_r\sum_{i=t_r}^{t_{r+1}-1}\sum_{k=1}^K\pi_k\twonorm{\vw_i^k - \bvw_{r+1}}^2,
    \end{aligned}
\end{equation}
where $\bDelta_i = \sum_{k=1}^K \pi_k (\bG_i^k - \nabla \gL_k(\vw_i^k))$.
\end{lemma}
Next lemma provides the upper bound for the discrepancy of local updates in Algorithm \ref{alg:FedDA}, and the proof is in Section \ref{proof:lemma_hetero_R}.
\begin{lemma}\label{lemma_hetero_R}
Under the conditions of Theorem \ref{thm_one_stage}, for any $t_r\leq t\leq t_{r+1}-1$, we have
\begin{equation*}
    \twonorm{\vw_i^k - \bvw_{r+1}} \leq \frac{4\alpha_r}{E \mu A_r/2 + \gamma E} \LRs{4\sqrt{E}\sigma\log(2/\delta) + E(H + \Lambda\rho)}
\end{equation*}
holds with with probability at least $1-\delta$.
\end{lemma}

\begin{proof}[Proof of Theorem \ref{thm_one_stage}]
With the choice $\gamma = \mu a^3$, the following summation is bounded by
\begin{align*}
    \sum_{r=0}^R \frac{16\alpha_r^3 E}{(\mu A_r + 2\gamma)^2} \leq \sum_{r=0}^R \frac{16 E (r+a)^6}{\mu^2 (r^3/3 + ar^2 + a^2r + a^3)^2} \leq \frac{144 E (R+1)}{\mu^2}.
\end{align*}
Plugging the conclusion of Lemma \ref{lemma_hetero_R} into \eqref{indu_R}, it follows that
\begin{equation}\label{hp_2_1}
    \begin{aligned}
    \sum_{r=0}^{R}\alpha_r\sum_{i=t_r}^{t_{r+1}-1}\sum_{k=1}^K\pi_k\twonorm{\vw_i^k - \bvw_{r+1}}^2 &\leq\sum_{r=0}^{R} \frac{16E \alpha_r^3}{(\mu A_r + 2\gamma )^2}\LRs{32 E^{-1}\sigma^2\log^2(2/\delta) + 2(\Lambda\rho + H)^2}\\
    &\leq \frac{144E(R+1)}{\mu^2}\LRs{32 E^{-1}\sigma^2\log^2(2/\delta) + 2(\Lambda\rho + H)^2}
    \end{aligned}
\end{equation}
holds with probability at least $1-\delta/3$.
Using the concentration inequality in Lemma \ref{lemma_Delta_concentration}, with probability at least $1-\delta/3$, we have
\begin{equation}\label{hp_2_2}
    \begin{aligned}
    \sum_{r=0}^R \alpha_r\sum_{i=t_r}^{t_{r+1}-1}\inprod{\bDelta_i}{\hvw - \bvw_{r}} &\leq \bar{\sigma}\sqrt{3\log(3/\delta)}\LRs{\sum_{r=0}^R\alpha_r^2\sum_{i=t_r}^{t_{r+1}-1} \twonorm{\hvw - \bvw_{r}}^2}^{1/2}\\
    &\leq \bar{\sigma}\sqrt{3\log(3/\delta)}\LRs{E\sum_{r=0}^R\alpha_r^2 \gR^2(\hvw - \bvw_{r})}^{1/2}\\
    &\leq \bar{\sigma}\sqrt{6E\log(3/\delta)}\LRs{\sum_{r=0}^R\alpha_r^2 (\gR^2(\hvw - \vw_{0}) + \gR^2(\bvw_r -\vw_0))}^{1/2}\\
    &\leq 2\epsilon_0\bar{\sigma}\sqrt{3\log(3/\delta)E \sum_{r=0}^R (r + a)^4}.
    \end{aligned}
\end{equation}
In the second inequality of \eqref{hp_2_2}, we used the assumption $\twonorm{\cdot}\leq \gR(\cdot)$. And the last inequality of \eqref{hp_2_2} follows from the constraint in proximal step $\gR(\bvw_r - \vw_0) \leq \epsilon_0$ and the assumption $\gR(\hvw - \vw_0) \leq \epsilon_0$. Additionally, by Lemma \ref{lemma_concen_norm}, with probability at least $1-\delta/3$, we also have
\begin{equation}\label{hp_2_3}
    \begin{aligned}
    \LRtwonorm{\sum_{i=t_r}^{t_{r+1}-1} \bDelta_i}^2 &\leq 8\log(6/(8\delta))\LRs{\sum_{i=t_r}^{t_{r+1}-1}\E_{\gP}\twonorm{\bDelta_i}^2 + \sum_{i=t_r}^{t_{r+1}-1}\twonorm{\bDelta_i}^2}\\
    &\leq 8\log(6/(8\delta))\LRs{2\sum_{i=t_r}^{t_{r+1}-1}\E_{\gP}\twonorm{\bDelta_i}^2 + \max\LRl{8\bar{\sigma}^2\log(6/\delta), 16\bar{\sigma}^2\sqrt{E \log(6/\delta)}}}\\
    &\leq  8\log(6/(8\delta))\LRs{2E \bar{\sigma}^2 + 16\bar{\sigma}^2 \sqrt{E\log(6/\delta)}}\\
    &\leq  16\bar{\sigma}^2\log(6/\delta)\LRs{E + 8 \sqrt{E\log(6/\delta)}},
    \end{aligned}
\end{equation}
where the second inequality follows from Lemma \ref{lemma_Delta_concentration} and the fact $\E_{\gD}[\twonorm{\bDelta_i}^2]\leq \bar{\sigma}^2$. Since $a \geq 4L/\mu$, it holds that
\begin{align*}
    \mu A_r + 2\gamma - 2 (L+\mu) \alpha_r \geq \frac{\mu r^3}{3} + 2 \mu a^3.
\end{align*}
In accordance with \eqref{hp_2_3}, for $R\geq a$, we have 
\begin{equation}\label{hp_2_4}
    \begin{aligned}
    \sum_{r=0}^R\alpha_r^2\frac{\twonorm{\sum_{i=t_r}^{t_{r+1}-1}\bDelta_i}^2}{2(A_r E\mu  + 2\gamma E - 2(L+\mu)E \alpha_r)}& =  16\bar{\sigma}^2\log(6/\delta)\LRs{E + 8 \sqrt{E\log(3/\delta)}} \sum_{r=0}^R \frac{(r+a)^4}{ E\mu (r^3+6a^3)}\\
    &\leq 16\bar{\sigma}^2\log(6/\delta)\LRs{1 + 8 \sqrt{\frac{\log(6/\delta)}{E}}} \frac{R(R+1)}{\mu}.
    \end{aligned}
\end{equation}
In addition, it follows from $D_{R+1}(\widehat{\vw};\bvw_{R+1})\geq 0$ and $D_{0}(\widehat{\vw};\bvw_{0}) = \gamma E\twonorm{\hvw-\vw_0}^2/2$ that
\begin{equation}\label{hp_2_5}
    \begin{aligned}
    \sum_{r=0}^R D_{r}(\widehat{\vw};\bvw_{r})- D_{r+1}(\widehat{\vw};\bvw_{r+1}) = D_{0}(\hvw; \bvw_{0}) - D_{R+1}(\hvw;\bvw_{R+1})
    \leq \frac{\gamma}{2} E\twonorm{\hvw-\vw_0}^2.
    \end{aligned}
\end{equation}
Telescoping the induction relation \eqref{indu_R} from $r=0$ to $r=R$, in conjunction with bounds \eqref{hp_2_1}-\eqref{hp_2_5}, we can guarantee with probability at least $1-\delta$
\begin{equation*}\label{avg_bound}
    \begin{aligned}
    \frac{1}{A_R}\sum_{r=0}^R\alpha_r[\phi\LRs{\bvw_{r+1}} - \phi(\hvw)] &\leq \frac{2\gamma \epsilon_0^2}{A_R} + 32\bar{\sigma}^2\log(6/\delta)\LRs{1 + 8 \sqrt{\frac{\log(3/\delta)}{E}}} \frac{R(R+1)}{\mu EA_R}\\
    &+\frac{144(3L+2\mu)(R+1)}{\mu^2 A_R}\LRs{32 E^{-1}\sigma^2\log^2(2/\delta) + 2(\Lambda\rho + H)^2}\\
    &+ \frac{8\epsilon_0\bar{\sigma}\sqrt{3\log(3/\delta)E\sum_{r=0}^R(r+a)^4}}{E A_R} + 20(\tau + \nu)\epsilon_0^2\\
    &\lesssim \frac{\mu a^3 E^3 \epsilon_0^2}{T^3} + \frac{\bar{\sigma}^2\log(1/\delta)}{\mu T} + \frac{L}{\mu^2 T^2}\LRs{ E\sigma^2\log^2(1/\delta) + E^2(\Lambda\rho + H)^2}\\
    &+ \frac{\epsilon_0 \bar{\sigma} \sqrt{\log(1/\delta)}}{\sqrt{T}} + (\tau + \nu)\epsilon_0^2.
\end{aligned}
\end{equation*}
In the last inequality, we also used $\sum_{r=0}^R(r+a)^4/R^6 \leq (R+a)^5/R^6 \lesssim 1/R$ since $R \geq a$. Therefore the conclusion follows from Jensen's inequality.
\end{proof}

\subsection{Proof of Theorem \ref{thm_multi_stage}}\label{appendix::multi_stage}
The following corollary is a direct result of Theorem \ref{thm_one_stage}.
\begin{corollary}\label{fast_corollary_simple}
Under the same conditions in Theorem \ref{thm_one_stage}, suppose the output of previous stage satisfies $\gR(\hvw_{m-1} - \hvw^m) \leq \epsilon_m$. We choose the number of local iterations $E_m$ such that $E_m^2 \lesssim \bar{\sigma}^2 \mu T_m/ (L(H + \Lambda\rho)^2)$ and $T_m \gtrsim \bar{\sigma}^2/(\mu^2 \epsilon_m^2)$ for $T_m = E_m R_m$, then the excess risk after the $m$-th stage is bounded by
\begin{equation*}
    \phi(\hvw_{m}) - \phi(\widehat{\vw}^{m}) \lesssim \frac{\bar{\sigma}\epsilon_m\sqrt{\log(1/\delta)}}{\sqrt{T_m}} + (\tau + \nu) \epsilon_m^2
\end{equation*}
with probability at least $1 - \delta$.
\end{corollary}
The next lemma restricts the averaged optimization error to a cone-like set. The conclusion \eqref{cone_conclusion} is a direct result from the relation \textcolor{blue}{(83)} in the supplementary material of \citet{agarwal2012fast} (see page 18) and the definition of $\Psi(\bar{\gM})$. And the conclusion is from $(90a)$ in the supplementary material of \citet{agarwal2012fast} (see page 21).
\begin{lemma}[Lemma 3 and 11, \citet{agarwal2012fast}, modified]\label{lemma_iter_cone}
Let $\widehat{\vw}$ be any optimum of the following regularized M-estimator
\begin{equation*}
    \min_{\vw \in \gW} \phi(\vw) :=\min_{\vw \in \gW}\LRl{\sum_{k=1}^K \pi_k \gL_k(\vw) + \lambda \gR(\vw)},
\end{equation*}
where $\lambda > \gR^{*}(\sum_{k=1}^K \pi_k\nabla \gL_k(\vw^{*}))/2$. Denote $v:=8\Psi(\bar{\gM})\twonorm{\hvw - \vw^{*}} + 2\eta/\lambda$. If $\phi(\vw) - \phi(\widehat{\vw}) \leq \eta$ for some $\eta > 0$ and $\vw^{*} \in \gW$, then we have
\begin{equation}\label{cone_conclusion}
    \gR(\vw - \vw^{*}) \leq 4 \Psi(\bar{\gM}) \twonorm{\vw - \vw^{*}} + 2\frac{\eta}{\lambda}
\end{equation}
and
\begin{equation}\label{eq:two_norm_cone}
    \LRs{\frac{\mu}{2} - 32\Psi^2(\bar{\gM})\tau} \twonorm{\vw - \widehat{\vw}}^2 \leq  2\tau v^2 + \phi(\vw) - \phi(\widehat{\vw}).
\end{equation}
for any $\gR$-decomposable subspace pair $(\gM, \bar{\gM}^{\top})$.
\end{lemma}
\begin{proof}[Proof of Theorem \ref{thm_multi_stage}]
In the first stage, we consider the following subproblem
\begin{equation}\label{problem-0}
    \hvw^{1} = \arg\min_{\vw \in \gW(\hvw_0;\epsilon_0)} \phi_1(\vw) := \arg\min_{\vw \in \gW(\hvw_0;\epsilon_0)} \LRl{\sum_{k=1}^K \pi_k \gL_k(\vw) + \lambda_0 \gR(\vw)}.
\end{equation}
Note that $\gR(\hvw_0-\hvw_{\text{opt}}) \leq 84 \Psi^2(\bar{\gM}) \lambda_0/\mu$, using Proposition \ref{thm_stat} we have
\begin{align*}
    \gR(\hvw_0 - \hvw^{1}) &\leq \gR(\hvw_0 - \hvw_{\text{opt}}) + \gR(\hvw_{\text{opt}} - \vw^{*}) + \gR(\hvw^{1} - \vw^{*})\\
    &\leq 84\Psi^2(\bar{\gM}) \frac{\lambda_0}{\mu} + 12\Psi^2(\bar{\gM}) \frac{\lambda_{\text{opt}}}{\mu} + 12\Psi^2(\bar{\gM}) \frac{\lambda_0}{\mu} \leq 108\Psi^2(\bar{\gM}) \frac{\lambda_0}{\mu} := \epsilon_0,
\end{align*}
where we used the fact $\lambda_0\geq \lambda_{\text{opt}}$. It means that $\widehat{w}^1$ is indeed feasible for problem \eqref{problem-0}. Choosing $R_0$ and $E_0$ in Algorithm \ref{alg:multi-FedDA} such that
\begin{equation*}
    T_0 = R_0 E_0 = \frac{8\times 54^2\Psi^4(\bar{\gM}) \bar{\sigma}^2\log(8M/\delta)}{\mu^2 \epsilon_0^2},
\end{equation*}
Corollary \ref{fast_corollary_simple} yields that with probability at least $1 - \delta/(2M)$
\begin{align*}
    \phi_1(\hvw_1) - \phi_1(\hvw^1) &\leq \frac{\bar{\sigma}\epsilon_0 \sqrt{\log(8M/\delta)}}{\sqrt{T_0}} + (\tau+ \nu)\epsilon_0^2\\
    &\leq \frac{\bar{\sigma}\epsilon_0 \sqrt{\log(8M/\delta)}}{\sqrt{T_0}} + \frac{\mu}{8\times 54^2\Psi^2(\bar{\gM})}\epsilon_0^2 \leq \frac{1}{\mu}\Psi^2(\bar{\gM})\lambda_0^2: = \eta_1,
\end{align*}
where we used the assumption $8\times 54^2\Psi^2(\bar{\gM})(\tau + \nu) \leq \mu$. In fact, $\vw^{*}$ is also feasible for \eqref{problem-0} since
\begin{align*}
    \gR(\vw^{*} - \hvw_0) &\leq \gR(\vw^{*} - \hvw_{\text{opt}}) + \gR(\hvw_{\text{opt}}-\hvw_{0})\\
    &\leq 12 \Psi^2(\bar{\gM}) \frac{\lambda_{\text{opt}}}{\mu} + 84 \Psi^2(\bar{\gM}) \frac{\lambda_0}{\mu}\leq \epsilon_0.
\end{align*}
In addition, $\lambda_0 > \lambda_{\text{opt}} > \gR^{*}(\nabla \gL(\vw^{*}))$ holds. Applying Lemma \ref{lemma_iter_cone}, we can obtain that
\begin{align}
    \gR(\hvw_{1} - \vw^{*}) &\stackrel{(a)}{\leq}  4\Psi(\bar{\gM}) \twonorm{\hvw_1 - \vw^{*}} + \frac{\eta_1}{\lambda_0}\nonumber\\
    &\leq 4\Psi(\bar{\gM}) \twonorm{\hvw_1 - \hvw^{1}} + 4\Psi(\bar{\gM}) \twonorm{\hvw^{1} - \vw^{*}} + \frac{\eta_1}{\lambda_0}\nonumber\\
    &\stackrel{(b)}{\leq} \frac{8\tau^{1/2} \Psi(\bar{\gM})}{\mu^{1/2}}\LRs{8\Psi(\bar{\gM})\twonorm{\hvw^{1} - \vw^{*}} + 2\frac{\eta_1}{\lambda_0}} + \frac{8\Psi(\bar{\gM})}{\mu^{1/2}}\eta_1^{1/2}  + 4\Psi(\bar{\gM}) \twonorm{\hvw^{1} - \vw^{*}} + \frac{\eta_1}{\lambda_0}\nonumber\\
    &\stackrel{(c)}{\leq} 4\Psi(\bar{\gM}) \twonorm{\hvw^{1} - \vw^{*}} + \frac{\eta_1}{\lambda_0} + \frac{8\Psi^2(\bar{\gM})\lambda_0}{\mu} + 4\Psi(\bar{\gM}) \twonorm{\hvw^{1} - \vw^{*}} + \frac{\eta_1}{\lambda_0}\nonumber\\
    &\stackrel{(d)}{\leq} 8\times 3\Psi^2(\bar{\gM})\frac{\lambda_0}{\mu} + \frac{10\eta_1}{\lambda_0} \leq 48 \Psi^2(\bar{\gM})\frac{\lambda_0}{\mu}.
    \label{mul-1}
\end{align}
The inequality $(a)$ and $(b)$ follow from \eqref{cone_conclusion} and \eqref{eq:two_norm_cone} in Lemma \ref{lemma_iter_cone} respectively. We also used the assumption $128\Psi^2(\bar{\gM})\tau \leq \mu$ in the inequality $(b)$ and $(c)$. The inequality $(d)$ follows from Proposition \ref{thm_stat}. Let $\lambda_m = \lambda_0\cdot 2^{-m}$, then we consider the following optimization problems
\begin{equation}\label{problem-m}
    \hvw^{m+1} = \arg\min_{\vw \in \gW(\hvw_m;\epsilon_m)} \phi_{m+1}(\vw) := \arg\min_{\vw \in \gW(\hvw_{m};\epsilon_m)} \LRl{\sum_{k=1}^K \pi_k \gL_k(\vw) + \lambda_m \gR(\vw)}
\end{equation}
where $\epsilon_m = 108\Psi^2(\bar{\gM})\lambda_{m}/\mu$ for $m\geq 0$. We define the following good events: for any $m=0,1,...,M-1$
\begin{equation*}
    \gA_{m} = \LRl{\gR(\hvw_{m+1} - \vw^{*}) \leq 48\frac{\Psi^2(\bar{\gM})\lambda_m}{\mu}}.
\end{equation*}
Now we prove $\sP(\gA_{m}^c)\leq \frac{\delta}{2}+\frac{m\delta}{2 M}$. Recall the definition of $\eta_0$, then it follows from \eqref{mul-1} that $\sP(\gA_0^c) \leq \delta/2$. Now we assume $\sP(\gA_{m-1}^c)\leq \frac{\delta}{2}+\frac{(m-1)\delta}{2 M}$ holds. Under the event $\gA_{m-1}$, note that
\begin{equation}\label{ini-dist}
    \begin{aligned}
    \gR(\hvw_{m} - \hvw^{m+1}) &\leq \gR(\hvw_{m} - \vw^{*}) + \gR(\hvw^{m+1} - \vw^{*}) \leq 48\Psi^2(\bar{\gM})\frac{\lambda_{m-1}}{\mu} + 12\Psi^2(\bar{\gM})\frac{\lambda_m}{\mu}\\
    &= 96\Psi^2(\bar{\gM})\frac{\lambda_{m}}{\mu} + 12\Psi^2(\bar{\gM})\frac{\lambda_{m}}{\mu} = \epsilon_m,
\end{aligned}
\end{equation}
where we applied Proposition \ref{thm_stat} to $\gR(\hvw^{m} - \vw^{*})$. We may choose $R_m$ and $E_m$ in Algorithm \ref{alg:multi-FedDA} satisfies that
\begin{equation*}
    T_{m} = R_m E_m = \frac{8\times 54^2 \Psi^4(\bar{\gM})\bar{\sigma}^2\log(2M/\delta)}{\mu^2 \epsilon_m^2},
\end{equation*}
then Corollary \ref{fast_corollary_simple} guarantees there exists some Borel set $\gC_m$ such that $\sP(\gC_m^c) \leq \delta/(2M)$. Under the event $\gA_{m-1}\cap\gC_m$, we have
\begin{align*}
    \phi_{m+1}(\hvw_{m+1}) - \phi_{m+1}(\hvw^{m+1}) &\leq \frac{\bar{\sigma}\epsilon_m\sqrt{\log(2M/\delta)}}{\sqrt{T_m}} + (\tau + \nu)\epsilon_m^2\\
    &\leq \frac{\bar{\sigma}\epsilon_m\sqrt{\log(1/\delta)}}{\sqrt{T_m}} + \frac{\mu}{8\times 54^2\Psi^2(\bar{\gM})}\epsilon_m^2 \leq \frac{\Psi^2(\bar{\gM})\lambda_m^2}{\mu} := \eta_{m+1}.
\end{align*}
In the first inequality above, we used $\hvw_{m-1}$ is the initial point of the $m$-th stage and the relation \eqref{ini-dist}. In the second inequality above, we used the Assumption \ref{assum_coef}. Recall that $\gR(\hvw_{m-1} - \vw^{*})\leq 48\Psi^2(\bar{\gM})\lambda_m/\mu \leq \epsilon_m$, thus $\vw^{*}$ is also feasible for problem \eqref{problem-m}. Applying Lemma \ref{lemma_iter_cone} again, similar to \eqref{mul-1}, we also have
\begin{equation}\label{mul-2}
    \begin{aligned}
    \gR(\hvw_{m+1} - \vw^{*}) \leq 24\Psi^2(\bar{\gM})\frac{\lambda_m}{\mu} + \frac{2\eta_{m+1}}{\lambda_m} \leq 48\Psi^2(\bar{\gM})\frac{\lambda_m}{\mu}.
    \end{aligned}
\end{equation}
Hence we have proved $(\gA_{m-1}\cap \gC_m)\subseteq \gA_m$, then it follows that
\begin{align*}
    \sP(\gA_m^c)\leq \sP(\gA_{M-1}^c) + \sP(\gC_{M}) \leq \frac{\delta}{2}+\frac{(m-1)\delta}{2 M} + \frac{\delta}{2M} = \frac{\delta}{2}+\frac{m\delta}{2 M}.
\end{align*}
We choose the number of stages such that $\lambda_{M-1} = \lambda_{\text{opt}}$, which means $M = \log_2(\lambda_0/\lambda_{\text{opt}}) + 1$. In fact, at the $M$-th stage, $\phi_{M}(\cdot) = \phi(\cdot)$ and $\hvw^{M} = \hvw$ since $\lambda_M = \lambda_{\text{opt}}$. Under the event $\gA_{M}$ with $\sP(\gA_{M}^c)\leq \delta$, we are guaranteed that
\begin{equation*}
    \begin{aligned}
    \phi(\hvw_M) - \phi(\hvw^M) = \phi(\hvw_M) - \phi(\hvw_{\text{opt}}) &\leq \frac{\Psi^2(\bar{\gM})\lambda_{\text{opt}}^2}{\mu}: = \eta_M.
    \end{aligned}
\end{equation*}
Together with the second conclusion in Lemma \ref{lemma_iter_cone}, we have
\begin{align*}
    \twonorm{\hvw_M - \hvw_{\text{opt}}}^2 &\leq \frac{6\tau}{\mu}\LRs{8\Psi(\bar{\gM})\twonorm{\hvw_{\text{opt}} - \vw^{*}} + \frac{\eta_M^2}{\lambda_{\text{opt}}^2}}^2 + \frac{3}{\mu}\eta_M\\
    &\leq \twonorm{\hvw_{\text{opt}} - \vw^{*}}^2 + \frac{6}{\mu}\eta_M = \twonorm{\hvw_{\text{opt}} - \vw^{*}}^2 + \frac{6\Psi^2(\bar{\gM})\lambda_{\text{opt}}^2}{\mu^2}.
\end{align*}
In addition, from the definition of $\gA_M$, we also have
\begin{equation*}
    \begin{aligned}
    \gR(\hvw_{M} - \vw^{*}) &\leq \frac{48\Psi^2(\bar{\gM})\lambda_M}{\mu^2} = \frac{48\Psi^2(\bar{\gM})\lambda_{\text{opt}}}{\mu^2}.
    \end{aligned}
\end{equation*}
Now we consider the total complexity,
\begin{align*}
    T &= \sum_{m=0}^{M-1}T_m = \sum_{m=0}^{M-1}\frac{16\times 54^2 \Psi^4(\bar{\gM})\bar{\sigma}^2\log(2M/\delta)}{\mu^2\epsilon_m^2}\\
    &\leq \sum_{m=0}^{M-1} \frac{4\bar{\sigma}^2\log(2M/\delta)}{\lambda_m^2} \leq M\cdot 2^{2M} \frac{4\bar{\sigma}^2\log(2M/\delta)}{ \lambda_{0}^2}\\
    &\leq  \frac{4\bar{\sigma}^2(\log_2(\lambda_0/\lambda_{\text{opt}})+1)}{\lambda_{\text{opt}}^2}\log\LRs{\frac{\log_2(\lambda_0/\lambda_{\text{opt}}) + 1}{\delta}}.
\end{align*}
\end{proof}

\section{One-Step Induction Relation}
\begin{lemma}[Proposition 1 in the appendix of \cite{chen2012optimal}]\label{lemma_dual_diff}
Given any proper lsc convex function $\psi(x)$ and a sequence of $\{\vz_i\}_{i=0}^t$ with each $\vz_i \in \gW$, if
\begin{equation*}
    \vz_{+} = \arg\min_{\vw \in \gW}\LRl{\psi(\vw) + \sum_{i = 0}^t \frac{\eta_i}{2}\twonorm{\vw - \vz_i}^2},
\end{equation*}
where $\{\eta_i\geq 0\}_{i=1}^t$ is a sequence of parameters, then for any $\vw \in \gW$:
\begin{equation}\label{dual_diff}
    \LRs{\frac{1}{2}\sum_{i=0}^t\eta_i}\twonorm{\vw - \vz_{+}}^2\leq \psi(\vw) + \sum_{i=0}^t \frac{\eta_i}{2}\twonorm{\vw - \vz_i}^2 - \LRl{\psi(\vz_{+}) + \sum_{i=0}^t \frac{\eta_i}{2}\twonorm{\vz_{+} - \vz_i}^2}.
\end{equation}
\end{lemma}
\subsection{Deferred Proof of Lemma \ref{lemma_one_step}}\label{proof:lemma_lemma_one_step}
\begin{proof}[Proof of Lemma \ref{lemma_one_step}]
According to the definition of $D_{t+1}(\hvw;\vw_{t+1})$ in \eqref{pseudo_dist}, we note that
\begin{align*}
    D_{t+1}(\widehat{\vw};\vw_{t+1})=& \inprod{\hvw-\vw_{t+1}}{\vg_{t}} + \frac{\mu}{4}\sum_{i=0}^{t}\alpha_i \LRs{\twonorm{\hvw-\cvw_i}^2 - \twonorm{\vw_{t+1}-\cvw_i}^2}\\
    + & \frac{\gamma}{2} \LRs{\twonorm{\hvw-\vw_0}^2 - \twonorm{\vw_{t+1}-\vw_0}^2} + A_t (h(\hvw)-h(\vw_{t+1})).
\end{align*}
Recall the fact $A_t = A_{t-1} + \alpha_t$, then simple arrangement gives rise to the following decomposition
\begin{align*}
    D_{t+1}(\widehat{\vw};\vw_{t+1})
    = &\inprod{\hvw-\vw_{t}}{\vg_{t-1}} + \frac{\mu}{4}\sum_{i=0}^{t}\alpha_i \LRs{\twonorm{\hvw-\cvw_i}^2 - \twonorm{\vw_{t}-\cvw_i}^2}\\
    + & \frac{\gamma}{2} \LRs{\twonorm{\hvw-\vw_0}^2 - \twonorm{\vw_{t}-\vw_0}^2} + A_{t-1} (h(\hvw)-h(\vw_{t}))\\
    - & \inprod{\vw_{t+1} - \vw_t}{\vg_{t-1}} - \frac{\mu}{4}\sum_{i=0}^{t-1}\alpha_i\LRs{\twonorm{\vw_{t+1} - \cvw_i}^2 - \twonorm{\vw_{t} - \cvw_i}^2}\\
    -& \frac{\gamma}{2} \LRs{\twonorm{\vw_{t+1} - \vw_0}^2 - \twonorm{\vw_{t} - \vw_0}^2}- A_{t-1}(h(\vw_{t+1}) - h(\vw_t)))\\
    +& \alpha_t\LRl{\inprod{\bG_t}{\widehat{\vw} - \vw_{t+1}} +  h(\widehat{\vw}) + \frac{\mu}{4} \twonorm{\widehat{\vw}-\cvw_t}^2 - \frac{\mu}{4} \twonorm{\vw_{t+1}-\cvw_t}^2 -  h(\vw_{t+1})}. 
\end{align*}
From the definitions of $D_{t}(\hvw;\vw_t)$ and $D_t(\vw_{t+1}; \vw_t)$ in \eqref{pseudo_dist}, together with $\gamma \geq \gamma$ we have
\begin{equation}\label{gen_distance_iter}
    \begin{aligned}
    D_{t+1}(\widehat{\vw};\vw_{t+1}) &\leq D_{t}(\widehat{\vw};\vw_{t}) - D_t(\vw_{t+1}; \vw_t) + \alpha_t\inprod{\bDelta_t}{\widehat{\vw} - \vw_{t+1}}\\
    & + \alpha_t\underbrace{\LRl{\gL(\vw_t) + \sum_{k=1}^K\pi_k\inprod{\nabla \gL_k(\vw_t^k)}{\widehat{\vw} - \vw_t} + \frac{\mu}{4}\twonorm{\cvw_t - \widehat{\vw}}^2 + h(\widehat{\vw})}}_{A_1}\\
    & - \alpha_t\underbrace{\LRl{\gL(\vw_t) + \sum_{k=1}^K\pi_k\inprod{\nabla \gL_k(\vw_t^k)}{\vw_{t+1} - \vw_t} + h(\vw_{t+1})}}_{A_2},
    \end{aligned}
\end{equation}
where $\bDelta_t = \sum_{k=1}^K\pi_k [\nabla f(\vw_t^k;\xi_t^k) - \nabla \gL_k(\vw_t^k)]$. By $\mu$-strong convexity and $L$-smoothness of local loss $\gL_k$, we get
\begin{align*}
    \gL_k(\vw_t^k) + \inprod{\nabla \gL_k(\vw_t^k)}{\widehat{\vw} - \vw_t^k} + \frac{\mu}{4}\twonorm{\widehat{\vw} - \cvw_t}^2&\leq \gL_k(\widehat{\vw}) + \frac{\mu}{4}\twonorm{\widehat{\vw} - \cvw_t}^2 - \frac{\mu}{2}\twonorm{\widehat{\vw} - \vw_t^k}^2\\
    &\leq \gL_k(\widehat{\vw}) + \frac{\mu}{2}\twonorm{\widehat{\vw} - \vw_t^k}^2+ \frac{\mu}{2}\twonorm{\cvw_t - \vw_t^k}^2- \frac{\mu}{2}\twonorm{\widehat{\vw} - \vw_t^k}^2\\
    & = \gL_k(\widehat{\vw}) + \frac{\mu}{2}\twonorm{\cvw_t - \vw_t^k}^2,
\end{align*}
and
\begin{align*}
    \gL_k(\vw_t) + \inprod{\nabla \gL_k(\vw_t^k)}{\vw_t^k - \vw_t} \leq \gL_k(\vw_t^k) + \frac{L}{2}\twonorm{\vw_t^k - \vw_t}^2.
\end{align*}
Summing the two inequalities above and taking average over $k$, together with the definition $\phi(\hvw) = \gL_(\hvw) + h(\hvw)$, we can bound $A_1$ in \eqref{gen_distance_iter} as
\begin{equation}\label{A1_bound}
    \begin{aligned}
    A_1 \leq & \phi(\widehat{\vw}) + \frac{\mu}{2}\sum_{k=1}^K\pi_k\twonorm{\vw_t^k - \cvw_t}^2 + \frac{L}{2}\sum_{k=1}^K\pi_k\twonorm{\vw_t^k - \vw_t}^2.
    \end{aligned}
\end{equation}
Using the convexity and $L$-smoothness of $\gL_k$ again, we can obtain that
\begin{align*}
    \gL_k(\vw_t) \geq \gL_k(\vw_t^k) + \inprod{\nabla \gL_k(\vw_t^k)}{\vw_t - \vw_t^k},
\end{align*}
and
\begin{equation*}
    \gL_k(\vw_t^k) \geq \gL_k(\vw_{t+1}) + \inprod{\nabla \gL_k(\vw_t^k)}{\vw_t^k - \vw_{t+1}} - \frac{L}{2}\twonorm{\vw_{t}^k - \vw_{t+1}}^2.
\end{equation*}
Summing two inequalities displayed above gives the bound of $A_2$, that is
\begin{equation}\label{A2_bound}
   \begin{aligned}
    A_2 &= \gL(\vw_t) + \sum_{k=1}^K\pi_k \inprod{\nabla \gL_k(\vw_t^k)}{\vw_{t+1} - \vw_t} + h(\vw_{t+1})\\ 
    &\geq \phi(\vw_{t+1}) - \frac{L}{2}\sum_{k=1}^K\pi_k\twonorm{\vw_t^k - \vw_{t+1}}^2\\
    &\geq \phi(\vw_{t+1}) - L\twonorm{\vw_t - \vw_{t+1}}^2 - L\sum_{k=1}^K\pi_k\twonorm{\vw_t^k - \vw_{t}}^2.
   \end{aligned}
\end{equation}
Plugging \eqref{A1_bound} and \eqref{A2_bound} into \eqref{gen_distance_iter} results in
\begin{equation}\label{mid_iter_bound}
     \begin{aligned}
     D_{t+1}(\widehat{\vw};\vw_{t+1}) \leq & D_{t}(\widehat{\vw};\vw_{t}) - D_{t}(\vw_{t+1};\vw_{t}) + \alpha_t[\phi(\widehat{\vw}) - \phi(\vw_{t+1})]\\
     + & \alpha_t L\twonorm{\vw_t - \vw_{t+1}}^2 + \inprod{\bDelta_t}{\widehat{\vw} - \vw_{t+1}} + \alpha_t\frac{\gamma - \gamma}{2}\twonorm{\hvw-\vw_0}^2\\
     + & \alpha_t\LRs{\frac{\mu}{2}\sum_{k=1}^K\pi_k\twonorm{\vw_t^k - \cvw_t}^2 + \frac{3L }{2}\sum_{k=1}^K\pi_k\twonorm{\vw_t^k - \vw_t}^2 + 4(4\tau + 3\nu)\epsilon_0^2}.
     \end{aligned}
\end{equation}
To apply Lemma \ref{lemma_dual_diff}, we let $\psi(\vw) = \inprod{\vw}{\vg_t}$, $\eta_i = \mu\alpha_i/2$ for $i\leq t-1$ and $\eta_t = \gamma/2$, $\vz_i = \cvw_i$ for $i\leq t-1$ and $\vz_t = \vw_0$. Recalling the definition of $\vw_t$ in \eqref{shadow_seq}, that is
\begin{equation*}
    \vw_t = \arg\min_{\vw \in \gW}\LRl{\psi(\vw) + \sum_{i = 0}^t \frac{\eta_i}{2}\twonorm{\vw - \vz_i}^2},
\end{equation*}
which implies that
\begin{align*}
    \LRs{\frac{\mu}{2}A_t + \gamma}\twonorm{\vw_{t+1} - \vw_t}^2&\leq \psi(\vw_{t+1}) + \sum_{i=0}^t \frac{\eta_i}{2}\twonorm{\vw_{t+1} - \vz_i}^2 - \LRl{\psi(\vw_{t}) + \sum_{i=0}^t \frac{\eta_i}{2}\twonorm{\vw_{t} - \vz_i}^2}\\
    &=D_t(\vw_{t+1};\vw_t).
\end{align*}
In addition, using the simple inequality: $-a x^2 + b x \leq \frac{b^2}{4a}$ for $a>0$, we have
\begin{equation*}
    \begin{aligned}
    &- D_{t}(\vw_{t+1};\vw_{t}) + L\alpha_t\twonorm{\vw_t - \vw_{t+1}}^2 + \alpha_t \inprod{\bDelta_t}{\widehat{\vw} - \vw_{t+1}}\\
    \leq & -\LRs{\frac{\mu}{2}A_t + \gamma - L\alpha_t}\twonorm{\vw_t - \vw_{t+1}}^2 + \alpha_t\twonorm{\bDelta_t}\twonorm{\vw_{t+1}-\vw_t} + \alpha_t\inprod{\bDelta_t}{\widehat{\vw} - \vw_{t}}\\
    \leq & \frac{\alpha_t^2\twonorm{\bDelta_t}^2}{2(\mu A_t + 2\gamma - 2L\alpha_t)} + \alpha_t\inprod{\bDelta_t}{\widehat{\vw} - \vw_{t}}.
    \end{aligned}
\end{equation*}
Then plugging above inequality into \eqref{mid_iter_bound} yields
\begin{equation}\label{mid_iter_bound_2}
    \begin{aligned}
      \alpha_t[\phi(\vw_{t+1})-\phi(\widehat{\vw})] \leq & D_{t}(\widehat{\vw};\vw_{t})- D_{t+1}(\widehat{\vw};\vw_{t+1})+\alpha_t\inprod{\bDelta_t}{\hvw - \vw_{t}} + \frac{\alpha_t^2\twonorm{\bDelta_t}^2}{2(A_t\mu  + 2\gamma - 2L\alpha_t)}\\
     +& \alpha_t\LRs{\frac{\mu}{2}\sum_{k=1}^K\pi_k\twonorm{\vw_t^k - \cvw_t}^2 + \frac{3L}{2}\sum_{k=1}^K\pi_k\twonorm{\vw_t^k - \vw_{t}}^2}.
     \end{aligned}
\end{equation}
Thus we have complete the proof of Lemma \ref{lemma_one_step}.
\end{proof}

\subsection{Deferred Proof of Lemma \ref{lemma_one_step_R}}\label{proof:lemma_one_step_R}
\begin{proof}[Proof of Lemma \ref{lemma_one_step_R}]
We first recall the definition of $D_{r+1}(\widehat{\vw};\bvw_{r+1})$
\begin{align*}
    D_{r+1}(\widehat{\vw};\bvw_{r+1})=& \inprod{\vg_{t_{r+1}-1}}{\hvw - \bvw_{r+1}} + \frac{\mu E}{4}\sum_{j=0}^{r}\alpha_j(\twonorm{\hvw - \bvw_{j}}^2 - \twonorm{\bvw_{r+1} - \bvw_{j}}^2)&\\
    & + \frac{\gamma E}{2}(\twonorm{\hvw - \bvw_0}^2 - \twonorm{\bvw_{r+1} - \bvw_0}^2) + A_r E[h(\hvw) - h(\bvw_{r+1})].
\end{align*}
Using $A_r = A_{r-1} + \alpha_r$, we may write $D_{r+1}(\widehat{\vw};\bvw_{r+1})$ as
\begin{align*}
    D_{r+1}(\widehat{\vw};\bvw_{r+1}) &= \inprod{\vg_{t_{r}-1}}{\hvw - \bvw_{r}} + \frac{\mu E}{4}\sum_{j=0}^{r-1}\alpha_j(\twonorm{\hvw - \bvw_{j}}^2 - \twonorm{\bvw_{r} - \bvw_{j}}^2)\\
    &+ \frac{\gamma E}{2}(\twonorm{\hvw - \bvw_0}^2 - \twonorm{\bvw_{r} - \bvw_0}^2)+ A_{r-1} E[h(\hvw) - h(\bvw_{r})]&\\
    &- \inprod{\vg_{t_{r}-1}}{\bvw_{r+1} - \bvw_{r}} - \frac{\mu E}{4}\sum_{j=0}^{r-1}\alpha_j(\twonorm{\bvw_{r+1} - \bvw_{j}}^2 - \twonorm{\bvw_{r} - \bvw_{j}}^2)\\
    &- \frac{\gamma E}{2}(\twonorm{\bvw_{r+1} - \bvw_0}^2 - \twonorm{\bvw_{r} - \bvw_0}^2)- A_{r-1} E[h(\bvw_{r+1}) - h(\bvw_{r})]\\
    &+ \inprod{\vg_{t_{r+1}-1} - \vg_{t_r-1}}{\hvw - \bvw_{r+1}}+ \alpha_r E\LRs{\frac{\mu}{4}\twonorm{\hvw - \bvw_r}^2 - \frac{\mu }{4}\twonorm{\bvw_{r+1} - \bvw_r}^2 + h(\hvw) - h(\bvw_{r+1})}.
\end{align*}
From the definition of $D_{r}(\widehat{\vw};\bvw_{r})$, $D_{r}(\bvw_{r};\bvw_{r+1})$ and $\vg_{t_{r+1}-1} - \vg_{t_r - 1} = \alpha_r\sum_{i = t_r}^{t_{r+1} - 1} \sum_{k=1}^K\pi_k\bG_i^k$, we have
\begin{equation}\label{gen_distance_iter_R}
    \begin{aligned}
    D_{r+1}(\widehat{\vw};\bvw_{r+1}) &\leq D_{r}(\widehat{\vw};\bvw_{r}) - D_r(\bvw_{r+1}; \bvw_r) + \alpha_r\sum_{i=t_r}^{t_{r+1}-1}\inprod{\bDelta_i}{\widehat{\vw} - \bvw_{r+1}}\\
    & + \alpha_r\underbrace{\sum_{i=t_r}^{t_{r+1}-1}\LRl{\gL(\bvw_r) + \sum_{k=1}^K\pi_k\inprod{\nabla \gL_k(\vw_i^k)}{\widehat{\vw} - \bvw_r} + \frac{\mu}{4}\twonorm{\bvw_r - \widehat{\vw}}^2 + h(\widehat{\vw})}}_{B_1}\\
    & - \alpha_r\underbrace{\sum_{i=t_r}^{t_{r+1}-1}\LRl{\gL(\bvw_r) + \sum_{k=1}^K\pi_k\inprod{\nabla \gL_k(\vw_i^k)}{\bvw_{r+1} - \bvw_r} + h(\bvw_{r+1})}}_{B_2},
    \end{aligned}
\end{equation}
where $\bDelta_i = \sum_{k=1}^K \pi_k (\bG_i^k - \nabla \gL_k(\vw_i^k))$. By the RSC and RSM of $\gL_k$, it follows that for any $ t_r\leq i \leq t_{r+1}-1$
\begin{align*}
    \gL_k(\vw_i^k) &+ \inprod{\nabla \gL_k(\vw_i^k)}{\widehat{\vw} - \vw_i^k} + \frac{\mu}{4}\twonorm{\widehat{\vw} - \bvw_r}^2\\
    \leq & \gL_k(\widehat{\vw}) + \frac{\mu}{4}\twonorm{\widehat{\vw} - \bvw_r}^2 - \frac{\mu}{2}\twonorm{\widehat{\vw} - \vw_i^k}^2 + \tau_k \gR^2(\widehat{\vw} - \vw_i^k)\\
    \leq & \gL_k(\widehat{\vw}) + \frac{\mu}{2}\twonorm{\bvw_{r} - \vw_i^k}^2 + 2\tau_k \gR^2(\widehat{\vw} - \bvw_{r+1}) + 2\tau_k \gR^2(\bvw_{r+1} - \vw_i^k)\\
    \leq & \gL_k(\widehat{\vw}) + \mu\twonorm{\bvw_{r+1} - \vw_i^k}^2 + \mu\twonorm{\bvw_{r} - \bvw_{r+1}}^2 + 2\tau_k \gR^2(\widehat{\vw} - \bvw_{r+1}) + 2\tau_k \gR^2(\bvw_{r+1} - \vw_i^k),
\end{align*}
and
\begin{align*}
    \gL_k(\bvw_r) + \inprod{\nabla \gL_k(\vw_i^k)}{\vw_i^k - \bvw_r} &\leq \gL_k(\vw_i^k) + \frac{L}{2}\twonorm{\vw_i^k - \bvw_r}^2 + \nu_k\gR^2(\vw_i^k - \bvw_r)\\
    &\leq  \gL_k(\vw_i^k) + L\twonorm{\bvw_{r+1} - \bvw_r}^2 + 2\nu_k\gR^2(\bvw_{r+1} - \bvw_{r})\\
    &+L\twonorm{\vw_i^k - \bvw_r}^2 + 2\nu_k\gR^2(\vw_i^k - \bvw_{r+1}).
\end{align*}
Summing the two inequalities above and taking average over $k$, together with the definition $\tau = \sum_{k=1}^K\pi_k \tau_k$, we can bound $B_1$ in \eqref{gen_distance_iter_R} as
\begin{equation}\label{B1_bound}
    \begin{aligned}
    B_1 \leq & E\phi(\widehat{\vw}) + E(L+\mu)\twonorm{\bvw_{r+1} - \bvw_r}^2 + 2E\nu \gR^2(\bvw_{r+1} - \bvw_{r}) + 2E\tau \gR^2(\widehat{\vw} - \bvw_{r+1})\\
    + & (L+\mu)\sum_{i=t_r}^{t_{r+1}-1}\sum_{k=1}^K\pi_k\twonorm{\vw_i^k - \bvw_{r+1}}^2 + \sum_{i=t_r}^{t_{r+1}-1}\sum_{k=1}^K\pi_k 2(\tau_k + \nu_k)\gR^2(\vw_i^k - \bvw_{r+1})\\
    \leq & E\phi(\widehat{\vw}) + (L+\mu)E\twonorm{\bvw_{r+1} - \bvw_r}^2 + (L+\mu)\sum_{i=t_r}^{t_{r+1}-1}\sum_{k=1}^K\pi_k\twonorm{\vw_i^k - \bvw_{r+1}}^2+16E(\tau + \nu)\epsilon_0^2.
    \end{aligned}
\end{equation}
We used the constrain $\gR(\vw - \bvw_0)\leq \epsilon_0$ in the proximal operator and the assumption $\gR(\hvw - \bvw_0)\leq \epsilon_0$ in the last inequality of \eqref{B1_bound}. Applying the convexity and RSM of $\gL_k$ again, we can obtain that
\begin{align*}
    \gL_k(\bvw_{r}) \geq \gL_k(\vw_i^k) + \inprod{\nabla \gL_k(\vw_i^k)}{\bvw_r - \vw_i^k},
\end{align*}
and
\begin{equation*}
    \gL_k(\vw_i^k) \geq \gL_k(\bvw_{r+1}) + \inprod{\nabla \gL_k(\vw_i^k)}{\vw_i^k - \bvw_{r+1}} - \frac{L}{2}\twonorm{\vw_{t}^k - \bvw_{r+1}}^2 - \nu_k\gR^2(\vw_{t}^k - \bvw_{r+1}).
\end{equation*}
In conjunction with the definition $\nu = \sum_{k=1}^K\pi_k\nu_k$, two inequalities displayed above shows that
\begin{equation}\label{B2_bound}
   \begin{aligned}
    B_2  &\geq E\phi(\bvw_{r+1}) - \frac{L}{2}\sum_{i=t_r}^{t_{r+1}-1}\sum_{k=1}^K\pi_k\twonorm{\vw_i^k - \bvw_{r+1}}^2 - \sum_{k=1}^K\pi_k\nu_k\gR^2(\vw_{i}^k - \bvw_{r+1})\\
    &\geq  E\phi(\bvw_{r+1}) - \frac{L}{2}\sum_{i=t_r}^{t_{r+1}-1}\sum_{k=1}^K\pi_k\twonorm{\vw_i^k - \bvw_{r+1}}^2 - 4E \nu \epsilon_0^2.
   \end{aligned}
\end{equation}
According to Lemma \ref{lemma_dual_diff}, we may guarantee that
\begin{equation}\label{Dr_bound}
    \begin{aligned}
    &- D_r(\bvw_{r+1}; \bvw_r) + (L+\mu)E\alpha_r\twonorm{\bvw_{r+1}- \bvw_r}^2 + \alpha_r\sum_{i=t_r}^{t_{r+1}-1}\inprod{\bDelta_i}{\widehat{\vw} - \bvw_{r+1}}\\
    \leq & -E\LRs{\frac{\mu}{2}A_r + \gamma - (L+\mu)\alpha_r}\twonorm{\bvw_{r+1}- \bvw_r}^2  -\alpha_r\sum_{i=t_r}^{t_{r+1}-1}\inprod{\bDelta_i}{\bvw_{r+1} - \bvw_r} + \alpha_r \sum_{i=t_r}^{t_{r+1}-1}\inprod{\bDelta_i}{\widehat{\vw} - \bvw_{r}}\\
    \leq & \frac{\alpha_r^2\twonorm{\sum_{i=t_r}^{t_{r+1}-1}\bDelta_i}^2}{2(A_r E \mu + 2\gamma E - 2(L+\mu)E)} + \alpha_r\sum_{i=t_r}^{t_{r+1}-1}\inprod{\bDelta_i}{\widehat{\vw} - \bvw_{r}},
    \end{aligned}
\end{equation}
where we used the inequality $-a x^2 + b x \leq \frac{b^2}{4a}$ for $a>0$ in the last inequality. Plugging three upper bounds \eqref{B1_bound}, \eqref{B2_bound} and \eqref{Dr_bound} into \eqref{gen_distance_iter_R}, we have
\begin{equation}\label{gen_distance_bound}
    \begin{aligned}
    &D_{r+1}(\widehat{\vw};\bvw_{r+1}) - D_{r}(\widehat{\vw};\bvw_{r})\\
    \leq & E\alpha_r[\phi(\hvw) - \phi(\bvw_{r+1})] + \alpha_r\sum_{i=t_r}^{t_{r+1}-1}\inprod{\bDelta_i}{\widehat{\vw} - \bvw_{r}}+\frac{\twonorm{\alpha_r^2\sum_{i=t_r}^{t_{r+1}-1}\bDelta_i}^2}{2(A_r E \mu + 2\gamma E - 2(L+\mu)E)}\\
    +& \alpha_r\frac{3L+2\mu}{2}\sum_{i=t_r}^{t_{r+1}-1}\sum_{k=1}^K\pi_k\twonorm{\vw_i^k - \bvw_{r+1}}^2 +  + 20\alpha_r E(\tau + \nu)\epsilon_0^2.
    \end{aligned}
\end{equation}
\end{proof}

\section{Upper Bound for Discrepancy}\label{appendix::discre}
\begin{lemma}[Proposition B.5, \cite{yuan2021federated}]\label{lemma_conjudate_smooth}
Let $\omega: \sR^d \to \sR \cup \{+\infty\}$ be a closed $\mu_{\omega}$-strongly convex function, for $\vz \in \sR^d$ we define
\begin{equation*}
    \nabla (\omega + h)^{*}(\vz) = \arg\min_{\vw} \LRl{\inprod{-\vz}{\vw} + \omega(\vw) + h(\vw)},
\end{equation*}
then it holds that
\begin{equation*}
    \|\nabla (\omega + h)^{*}(\vz) -\nabla (\omega + h)^*(\vy)\| \leq 1/\mu_{\omega} \|\vz-\vy\|_{*}.
\end{equation*}
\end{lemma}

\subsection{Deferred Proof of Lemma \ref{lemma_hetero_E}}\label{appen:discre_1}
\begin{proof}[Proof of Lemma \ref{lemma_hetero_E}]
Reacll the definitions of $\vw_t^k$ and $\vw_t$
\begin{equation*}
    \vw_t^k = \arg\min_{\vw\in \gW(\epsilon_0;\vw_0)}\LRl{\inprod{\vw}{\vg_{t-1}^k - \frac{\mu}{2}\tvw_{t-1}^k - \gamma \vw_0} + \LRs{\frac{A_{t-1}\mu}{2} + \gamma}\frac{\twonorm{\vw}^2}{2} + A_{t-1} h(\vw)}
\end{equation*}
and
\begin{equation*}
    \vw_t = \arg\min_{\vw\in \gW(\epsilon_0;\vw_0)}\LRl{\inprod{\vw}{\vg_{t-1} - \frac{\mu}{2}\tvw_{t-1} - \gamma \vw_0} + \LRs{\frac{A_{t-1}\mu}{2} + \gamma}\frac{\twonorm{\vw}^2}{2} + A_{t-1} h(\vw)}.
\end{equation*}
Since the synchronization at step $t_r$, we have $\vg_{t-1}^k - \vg_{t-1} = \sum_{i = t_r}^{t-1}\alpha_i (\bG_i^k - \sum_{l=1}^K\pi_l \bG_i^l)$ and $\tvw_{t-1}^k-\tvw_{t-1} = \sum_{i = t_r}^{t-1}\alpha_i(\vw_i^k - \sum_{l=1}^K\pi_l \vw_i^l)$ for $t_r\leq t-1 \leq t_{r+1}-1$. Then applying Lemma \ref{lemma_conjudate_smooth}, it holds that
\begin{equation}\label{theta_dp_bound}
    \begin{aligned}
    \twonorm{\vw_t^k - \vw_t} &\leq \frac{1}{\mu A_{t-1}/2 + \gamma} \LRs{\twonorm{\vg_{t-1}^k - \vg_{t-1}} + \frac{\mu}{2}\twonorm{\tvw_{t-1}^k -\tvw_{t-1}}}\\
    &\leq \frac{1}{\mu A_{t-1}/2 + \gamma}\LRs{\LRtwonorm{\sum_{i=t_r}^{t-1} \alpha_i(\bG_{i} - \bG_i^k)} + \frac{\mu}{2}\LRtwonorm{\sum_{i = t_r}^{t-1}\alpha_i(\vw_i^k - \cvw_i)}}\\
    &\leq \frac{1}{\mu A_{t-1}/2 + \gamma}\LRs{\sum_{l=1}^K \pi_l \LRtwonorm{\sum_{i=t_r}^{t-1} \alpha_i(\bG_{i}^l - \bG_i^k)} + \mu \rho (A_{t-1} - A_{t_r-1})},
    \end{aligned}
\end{equation}
where we used $\rho$-bounded domain in the last inequality. Let $\bDelta_i^k = \bG_i^k - \nabla\gL_k(\vw_i^k)$, then we may decompose the difference of local stochastic gradients as
\begin{equation}\label{G_bound_1}
    \begin{aligned}
    \LRtwonorm{\sum_{i=t_r}^{t-1}\alpha_i (\bG_i^l - \bG_i^k)}& \leq \LRtwonorm{\sum_{i=t_r}^{t-1} \alpha_i\bDelta_i^k} + \LRtwonorm{\sum_{i=t_r}^{t-1} \alpha_i\bDelta_i^l}+ \sum_{i=t_r}^{t-1}\alpha_i\LRtwonorm{\nabla \gL_l(\vw_i^l) - \nabla \gL_k(\vw_i^k)}\\
    &\leq \LRtwonorm{\sum_{i=t_r}^{t-1} \alpha_i\bDelta_i^k} + \LRtwonorm{\sum_{i=t_r}^{t-1} \alpha_i\bDelta_i^l} + \sum_{i=t_r}^{t-1}\alpha_i\LRtwonorm{\nabla \gL_l(\vw_i^l) - \nabla \gL(\vw_i^l)}\\
    & + \sum_{i=t_r}^{t-1}\alpha_i\LRtwonorm{\nabla \gL_k(\vw_i^k) - \nabla \gL(\vw_i^k)} + \sum_{i=t_r}^{t-1}\alpha_i\LRtwonorm{\nabla \gL(\vw_i^l) - \nabla \gL(\vw_i^k)}\\
    & \leq \LRtwonorm{\sum_{i=t_r}^{t-1} \alpha_i\bDelta_i^k} + \LRtwonorm{\sum_{i=t_r}^{t-1} \alpha_i\bDelta_i^l} + 2(A_{t-1} - A_{t_r - 1})(H + \Lambda\rho),
\end{aligned}
\end{equation}
where the third inequality follows from the bounded heterogeneity assumption and $\Lambda$-smoothness of global loss $\gL$.
By the conditional dependence, we have
\begin{align*}
    \E_{\gD}\LRm{\LRtwonorm{\sum_{i=t_r}^{t-1}\alpha_i \bDelta_i^k}^2|\gF_{t_r}} = \sum_{i=t_r}^{t-1}\alpha_i^2\E_{\gD}[\E(\twonorm{\bDelta_i^k}^2|\gF_{i})|\gF_{t_r}] \leq E\alpha_t^2 \sigma^2
\end{align*}
Taking expectation on both sides of \eqref{G_bound_1} and using the relation above, we have
\begin{equation}\label{hetero_E_2}
    \begin{aligned}
    \E_{\gD}\LRm{\LRtwonorm{\sum_{i=t_r}^{t-1} \alpha_i(\bG_{i}^l - \bG_i^k)}^2\Big| \gF_{t_r}} &\leq 2\E_{\gD}\LRm{\LRtwonorm{\sum_{i=t_r}^{t-1} \alpha_i\bDelta_i^k}^2 + \LRtwonorm{\sum_{i=t_r}^{t-1} \alpha_i\bDelta_i^l}^2\Big| \gF_{t_r}} + 4(A_{t-1} - A_{t_r - 1})^2(H + \Lambda\rho)^2\\
    &\leq 4 E\alpha_t^2 \sigma^2 + 4\alpha_t^2 E^2(H + \Lambda\rho)^2,
    \end{aligned}
\end{equation}
where the last inequality follows from $(A_{t-1} - A_{t_r-1})/\alpha_t\leq E$.
Combining \eqref{theta_dp_bound} and \eqref{hetero_E_2}, together with $\twonorm{\vw_i^k-\vw_i^l} \leq 2\rho$ and $\mu \leq \Lambda$, we are guaranteed that
\begin{align*}
    \E_{\gD}[\twonorm{\vw_t^k - \vw_t}^2] &\leq \frac{4 E\sigma^2\alpha_t^2}{(\mu A_t/2 + \gamma)^2} + \frac{4\alpha_t^2 E^2((H + \Lambda\rho)^2+\mu^2\rho^2)}{(\mu A_t/2 + \gamma)^2}\\
    &\leq \frac{4 E\sigma^2\alpha_t^2}{(\mu A_t/2 + \gamma)^2} + \frac{8\alpha_t^2 E^2(H + \Lambda\rho)^2}{(\mu A_t/2 + \gamma)^2}.
\end{align*}
Similarly, $\E_{\gD}[\twonorm{\cvw_t - \vw_t^k}^2]$ shares the same upper bound with $\E[\twonorm{\vw_t^k - \vw_t}^2]$ due to the following relation
\begin{align*}
    \twonorm{\vw_t^k - \tvw_t}&\leq \sum_{l=1}^{K}\pi_l\twonorm{\vw_{t}^{k} - \vw_t^{l}}\\
    &\leq \frac{1}{\mu A_{t-1}/2 + \gamma}\sum_{l=1}^K\pi_l\LRs{\LRtwonorm{\sum_{i=t_r}^{t-1} \alpha_i(\bG_i^k)-\bG_{i}^{l}} + \frac{\mu}{2}\LRtwonorm{\sum_{i = t_r}^{t-1}\alpha_i(\vw_i^k - \vw_i^{l})}}.
\end{align*}
\end{proof}

\subsection{Deferred Proof of Lemma \ref{lemma_hetero_R}}\label{proof:lemma_hetero_R}
\begin{proof}[Proof of Lemma \ref{lemma_hetero_R}]
Recalling the definitions of $\bvw_r$ and $\vw_i^k$ for $t_r \leq i \leq t_{r+1}-1$:
\begin{align*}
    \bvw_{r+1} &= \arg\min_{\vw \in \gW(\epsilon_0;\vw_0)}\LRl{\inprod{\vw}{\vg_{t_{r+1} - 1} - \frac{\mu E}{2}\sum_{j=0}^{r}\alpha_j\bvw_j - \gamma E \bvw_0} + \LRs{\frac{A_r \mu}{2} + \gamma}E\frac{\twonorm{\vw}^2}{2} + A_r E h(\vw)}\\
    \vw_i^k &= \arg\min_{\vw \in \gW(\epsilon_0;\vw_0)}\LRl{\inprod{\vw}{\vg_{i-1}^k - \frac{\mu E}{2}\sum_{j=0}^{r}\alpha_j\bvw_j - \gamma E \bvw_0} + \LRs{\frac{A_r \mu}{2} + \gamma}E\frac{\twonorm{\vw}^2}{2} + A_r E h(\vw)},
\end{align*}
where $\vg_{t_{r+1}-1} = \vg_{t_r-1} + \sum_{j=t_r}^{t_{r+1}-1} \alpha_r\bG_j$ and $\vg_{i}^k = \vg_{t_r-1} + \sum_{j=t_r}^{i} \alpha_r\bG_j^k$.
Using Lemma \ref{lemma_conjudate_smooth} and similar decomposition in \eqref{G_bound_1}, we have
\begin{equation}\label{hp_delta_0}
    \begin{aligned}
    \twonorm{\vw_i^k - \bvw_{r+1}} &\leq \frac{1}{A_r E\mu/2 + \gamma E} \twonorm{\vg_{i-1}^k-\vg_{t_{r+1}-1}}\leq \frac{1}{A_r E\mu/2 + \gamma E}\LRtwonorm{\sum_{j=i}^{t_{r+1}-1} \alpha_r(\bG_j - \bG_j^k)}\\
    &\leq \frac{1}{A_r E\mu/2 + \gamma E}\sum_{l=1}^K\pi_l\LRtwonorm{\sum_{j=i}^{t_{r+1}-1} \alpha_r(\bG_j^l - \bG_j^k)}\\
    &\leq \frac{\alpha_r}{A_r E\mu/2 + \gamma E}\sum_{l=1}^{K}\pi_l \LRs{\left\|\sum_{j=i}^{t_{r+1}-1}\bDelta_j^k\right\|+\left\|\sum_{j=i}^{t_{r+1}-1}\bDelta_j^l\right\| + 2(t_{r+1}-i)(H+\Lambda \rho)}.
    \end{aligned}
\end{equation}
Applying Lemma \ref{lemma_concen_norm}, we can obtain that
\begin{equation}\label{hp_delta_1}
    \LRtwonorm{\sum_{j=i}^{t_{r+1}-1} \bDelta_j^k} \leq 2\sqrt{2\log(1/(4\delta))}\LRl{\LRs{\sum_{j=i}^{t_{r+1}-1}\E_{\gD}\twonorm{\bDelta_j^k}^2}^{1/2} + \LRs{\sum_{j=i}^{t_{r+1}-1}\twonorm{\bDelta_j^k}^2}^{1/2}}
\end{equation}
holds with probability at least $1-\delta/2$. Using the light-tailed assumption and Lemma \ref{lemma_Delta_concentration}, we are guaranteed that $\E_{\gD}\twonorm{\bDelta_j^k}^2 \leq \sigma^2$ and
\begin{equation}\label{hp_delta_2}
    \sum_{j=i}^{t_{r+1}-1}\twonorm{\bDelta_j^k}^2\leq E\sigma^2 + \max\LRl{8\sigma^2 \log(2/\delta), 16\sigma^2 \sqrt{E\log(2/\delta)}}
\end{equation}
holds with probability at least $1-\delta/2$. Substituting \eqref{hp_delta_1} and \eqref{hp_delta_2} into \eqref{hp_delta_0}, it holds that
\begin{align*}
    \twonorm{\vw_i^k - \bvw_{r+1}} &\leq \frac{2\alpha_r}{E \mu A_r/2 + \gamma E} \LRs{\sqrt{2\log(1/(4\delta))}\LRs{2\sqrt{E}\sigma + 4\sqrt{E}\sigma\sqrt{\log(2/\delta)}} + 2E(H + \Lambda\rho)}\\
    &\leq \frac{2\alpha_r}{E \mu A_r/2 + \gamma E} \LRs{8\sqrt{E}\sigma\log(2/\delta) + 2E(H + \Lambda\rho)}
\end{align*}
with probability at least $1-\delta$.
\end{proof}

\section{Proof of Lemma \ref{lemma_concen_norm}}\label{appendix::shao}
This lemma is a martingale's version of Lemma 3.1 in \cite{he2000parameters}, we provide the proof for completeness.
\begin{proof}[Proof of Lemma \ref{lemma_concen_norm}]
Without loss of generality, we consider that $x > 16$. If $x \leq 16$, we may modify the tail probability in Lemma \ref{lemma_concen_norm} as $100 \exp(-x^2/100)$. Let $\{\zeta_i\}_{i=1}^{\infty}$ be an independent copy of $\LRl{\xi_i}_{i=1}^{\infty}$, which is also adapted to $\LRl{\gF_i}_{i=1}^{\infty}$. By Chebyshev's inequality, we have
\begin{equation}\label{concen_norm_1}
    \begin{aligned}
    \sP\LRs{\LRtwonorm{\sum_{i=1}^{t} \zeta_i}\leq 2B_t,\ \sum_{i=0}^t \twonorm{\zeta_i}^2 \leq 2 B_t^2}\geq & 1 - \sP\LRs{\LRtwonorm{\sum_{i=1}^{t} \zeta_i}> 2B_t} - \sP\LRs{\sum_{i=0}^t \twonorm{\zeta_i}^2 > 2 B_t^2}\\
    \geq & 1 - \sP\LRs{\LRtwonorm{\sum_{i=1}^{t} \zeta_i}> 2B_t} - \frac{1}{2}\\
    \geq & 1 - \frac{\E\LRtwonorm{\sum_{i=1}^{t} \zeta_i}^2}{4 B_t^2} - 1/2\\
    \geq & 1 - 1/4 - 1/2 =1/4,
    \end{aligned}
\end{equation}
where the last inequality follows from $B_t^2 = \E\twonorm{\sum_{i=1}^{t} \zeta_i}^2 = \sum_{i=1}^{t} \E\twonorm{\zeta_i}^2$ due to the martingale property. Let $\LRl{\varepsilon_i}_{i=1}^{t}$ be a Rademacher sequence independent of $\LRl{\xi_i}_{i=1}^t$ and $\LRl{\zeta_i}_{i=1}^t$. With slightly abusing notation, we denote $S_t = (\sum_{i=1}^t (\twonorm{\xi_i - \zeta_i}^2))^{1/2}$. We assume the following event holds
\begin{align*}
    \LRl{\LRtwonorm{\sum_{i=1}^{t} \xi_i}\geq x\LRs{B_t + \LRs{\sum_{i=1}^t \twonorm{\xi_i}^2}^{1/2}},\ \LRtwonorm{\sum_{i=1}^{t} \zeta_i}\leq 2B_t,\ \sum_{i=0}^t \twonorm{\zeta_i}^2 \leq 2 B_t^2}.
\end{align*}
Then we notice that
\begin{align*}
    \LRtwonorm{\sum_{i=1}^t \xi_i - \zeta_i}& \geq \LRtwonorm{\sum_{i=1}^t \xi_i} - \LRtwonorm{\sum_{i=1}^t \zeta_i}\geq x\LRs{B_t + \LRs{\sum_{i=1}^t \twonorm{\xi_i}^2}^{1/2}} - 2 B_t\\
    &\geq (x - 2)B_t + \frac{x}{2}\LRs{\sum_{i=1}^t \twonorm{\xi_i}^2}^{1/2} + \frac{x}{2}\LRs{\sum_{i=1}^t \twonorm{\zeta_i}^2}^{1/2} - \frac{\sqrt{2}x}{2} B_t\\
    &\stackrel{(a)}{\geq} \LRs{\frac{x}{4} - 2}B_t + \frac{x}{2}\LRs{\sum_{i=1}^t \twonorm{\xi_i - \zeta_i}^2}^{1/2}\\
    &\stackrel{(b)}{\geq} \frac{x}{2}\LRs{\sum_{i=1}^t \twonorm{\xi_i - \zeta_i}^2}^{1/2},
\end{align*}
where the inequality $(a)$ follows from the triangle inequality and the inequality $(b)$ holds since $x > 16$. The relation above implies that
\begin{equation}\label{concen_norm_2}
    \begin{aligned}
    \LRl{\LRtwonorm{\sum_{i=1}^{t} \xi_i}\geq x\LRs{B_t + \LRs{\sum_{i=1}^t \twonorm{\xi_i}^2}^{1/2}},\ \LRtwonorm{\sum_{i=1}^{t} \zeta_i}\leq 2B_t,\ \sum_{i=0}^t \twonorm{\zeta_i}^2 \leq 2 B_t^2}&\\
    \subseteq  \LRl{\LRtwonorm{\sum_{i=1}^t  \xi_i - \zeta_i}\geq \frac{x}{2}S_t}&.
\end{aligned}
\end{equation}
Using the dependence of $\xi_i$ and $\zeta_i$, we have
\begin{equation}\label{concen_norm_3}
    \begin{aligned}
    &\sP\LRs{\LRtwonorm{\sum_{i=1}^{t} \xi_i}\geq x\LRs{B_t + \LRs{\sum_{i=1}^t \twonorm{\xi_i}^2}^{1/2}}}\\
    = & \sP\LRs{\LRtwonorm{\sum_{i=1}^{t} \xi_i}\geq x\LRs{B_t + \LRs{\sum_{i=1}^t \twonorm{\xi_i}^2}^{1/2}},\ \LRtwonorm{\sum_{i=1}^{t} \zeta_i}\leq 12B_t,\ \sum_{i=0}^t \twonorm{\zeta_i}^2 \leq 2 B_t^2}\\
    \times & \sP\LRs{\LRtwonorm{\sum_{i=1}^{t} \zeta_i}\leq 12B_t,\ \sum_{i=0}^t \twonorm{\zeta_i}^2 \leq 2 B_t^2}\\
    \leq & 4 \sP\LRs{\LRtwonorm{\sum_{i=1}^t \xi_i - \zeta_i}\geq \frac{x}{2}S_t},
\end{aligned}
\end{equation}
where the last inequality follows from \eqref{concen_norm_1} and \eqref{concen_norm_2}. Note that $\LRl{\xi_i - \zeta_i}_{i=1}^{t}$ is a symmetric martingale difference sequence. Then using double expectation (given $\gF_t$, $\xi_i$ and $\zeta_i$ for $1\leq i\leq t$ are fixed), we have
\begin{align*}
    \sP\LRs{\LRtwonorm{\sum_{i=1}^t \xi_i - \zeta_i}\geq \frac{x}{2}S_t} &= \sP\LRs{\LRtwonorm{\sum_{i=1}^t (\xi_i - \zeta_i)\varepsilon_i}\geq \frac{x}{2}S_t}\\
    &= \E\LRl{\sP\LRs{\LRtwonorm{\sum_{i=1}^t (\xi_i - \zeta_i)\varepsilon_i}\geq \frac{x}{2}S_t\big | \gF_t}}\\
    &\leq \E\LRl{2\exp\LRs{- \frac{x^2 S_t^2}{8\sum_{i=1}^t\twonorm{\xi_i - \zeta_i}^2}}}\\
    &\leq 2\exp\LRs{-\frac{x^2}{8}},
\end{align*}
where the first inequality follows from the exponential inequality for Rademacher sequence (see, e.g., \citet{ledoux1991probability}, p.101). Plugging this upper bound into \eqref{concen_norm_3}, we are guaranteed that
\begin{align*}
    \sP\LRs{\LRtwonorm{\sum_{i=1}^{t} \xi_i}\geq x\LRs{B_t + \LRs{\sum_{i=1}^t \twonorm{\xi_i}^2}^{1/2}}}\leq 8\exp\LRs{-\frac{x^2}{8}}.
\end{align*}
\end{proof}

\begin{figure*}[ht]
    \centering
    \includegraphics[width=0.9\linewidth]{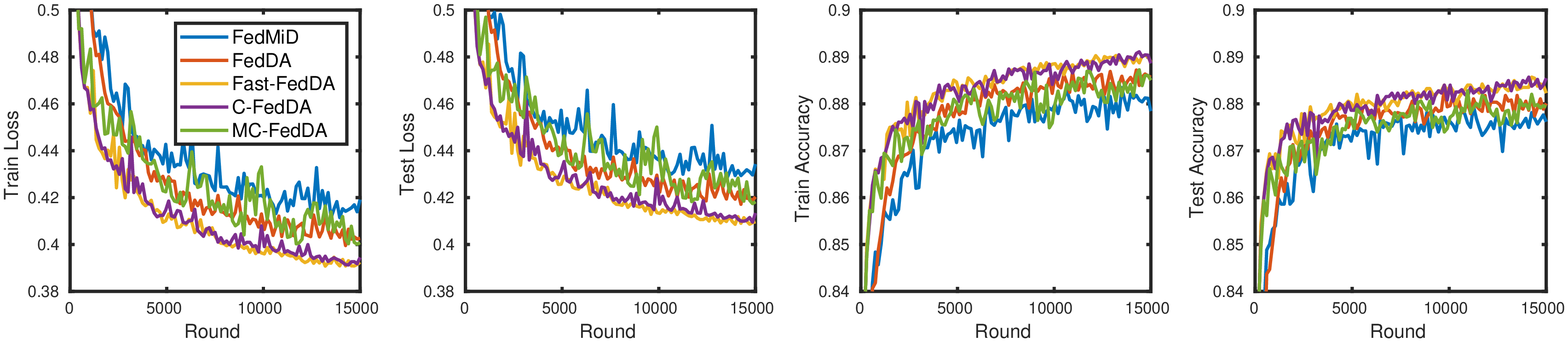}
    \caption{Results for federated sparse logistic regression on EMNIST-10 dataset. Our proposed algorithms \texttt{Fast-FedDA} and \texttt{C-FedDA} reach a lower loss and higher accuracy than the two baselines from \cite{yuan2021federated}, and exhibit faster convergence.}
    \label{fig:emnist_10}
\end{figure*}
\begin{figure*}[ht]
    \centering
    \includegraphics[width=0.9\linewidth]{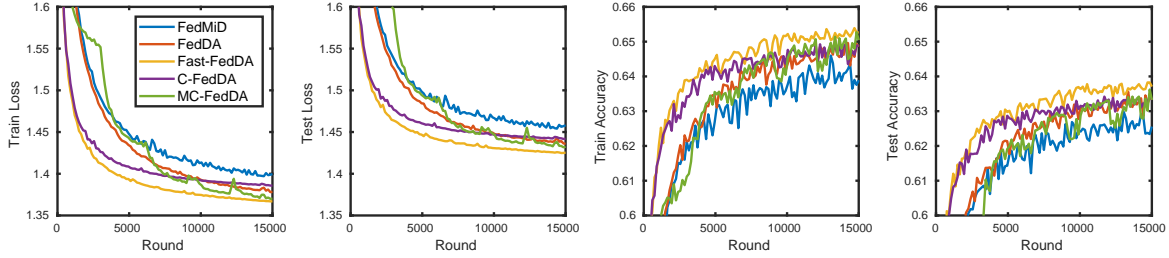}
    \caption{Results for federated sparse logistic regression on EMNIST-62 dataset. Our proposed algorithms \texttt{Fast-FedDA} and \texttt{MC-FedDA} reach a lower loss and higher accuracy than the two baselines from \cite{yuan2021federated}, and \texttt{Fast-FedDA} exhibits faster convergence.}
    \label{fig:emnist_62}
\end{figure*}

\section{Additional Results in Section \ref{sec:experiment}}\label{sec:realexp}
\paragraph{Federated sparse linear regression.}
For \texttt{MC-FedDA}, we set the number of stages $M = 3$ and use the regularization sequence $\{0.5^3, 0.5^4, 0.5^5\}$ for regularization parameters in 3 stages. For other methods, the regularization parameter is $\lambda = 0.5^5$. The hyperparameters for \texttt{Fast-FedDA} are $\mu = 0.1$ and $L = 550$. For \texttt{C-FedDA} and \texttt{MC-FedDA}, we choose $\mu = 0.1$ and $L = 600$. For \texttt{FedDA} and \texttt{FedMiD}, we set the server learning rate $\eta_s = 1.0$ and tuned the client learning rate $\eta_c$ by selecting the best performing value over the set $\{0.0001, 0.001, 0.01, 0.1\}$, which was $0.001$ for both baselines.
\paragraph{Federated low-rank matrix estimation.}
For \texttt{MC-FedDA}, we set the number of stages $M = 3$ and use the sequence $\{0.3, 0.15, 0.1\}$ for regularization parameters in 3 stages. For other methods, the regularization parameter is $\lambda = 0.1$. The choices for hyperparameters follow the same setting in sparse linear regression.

\paragraph{Federated sparse logistic regression.}
The experimental results on EMNIST-10 and EMNIST-62 are reported in Figure~\ref{fig:emnist_10} and Figure~\ref{fig:emnist_62} respectively. For the two baselines (\texttt{FedMid} and \texttt{FedDA}), we set the server learning rate $\eta_s = 1.0$ and tuned the client learning rate $\eta_c$ by selecting the best performing value over the set $\{0.001, 0.003, 0.01, 0.03, 0.1\}$, which was $\eta_c = 0.01$ for both baselines. For our proposed algorithms, we tuned $\mu$ and $\gamma$ by selecting the best performing values over the sets $\{0.0001, 0.0005, 0.001, 0.005, 0.01\}$ and $\{10, 25, 50, 100\}$, respectively. The best values were $\mu = 0.001$ and $\gamma = 25$ for all proposed algorithms. For \texttt{MC-FedDA}, we use the regularization sequence $\{0.000225, 0.00015, 0.0001, 0.0001, 0.0001\}$.

\end{document}